%% file: jml_main.tex
\documentclass[mathpazo,online]{jml}

\renewcommand\thisvolume{}
\renewcommand\thisnumber{}
\renewcommand\thisyear {}
\renewcommand\thismonth{}

\renewcommand\originalreceiveddate{xx/xx/xxxx}

\renewcommand\acceptdate{xx/xx/xxxx}

\renewcommand\pagestart{x}
\renewcommand\pageend{xx}

\renewcommand\articledoi{doi.org/10.4208/jml.xxx}

\renewcommand\handler{xxx}

\usepackage{lipsum}
\usepackage[utf8]{inputenc}
\usepackage[T1]{fontenc}    
\usepackage{hyperref}       
\usepackage{url}            
\usepackage{booktabs}       
\usepackage{amsfonts}       
\usepackage{nicefrac}       
\usepackage{microtype}      
\usepackage{xcolor}         
\usepackage{graphicx} 
\usepackage{braket}
\usepackage{algorithm}
\usepackage{algpseudocode}
\usepackage{float}
\usepackage{amsmath,amsthm,amssymb}
\usepackage{bm}
\usepackage{verbatim}
\usepackage{subcaption} 
\usepackage{enumitem,kantlipsum}

\usepackage{subcaption}

\usepackage{xr}
\makeatletter
\newcommand*{\addFileDependency}[1]{
\typeout{(#1)}
\IfFileExists{#1}{}{\typeout{No file #1.}}
}\makeatother

\newcommand{\rank}{{\mathrm{rank}}}
\newcommand{\RE}{{\mathrm{RE}}}
\newcommand{\DC}{{\mathrm{DC}}}
\newcommand{\supp}{{\mathrm{supp}}}

\newtheorem{assumption}{Assumption}
\DeclareMathOperator*{\argmin}{arg\,min}

\hypersetup{
   colorlinks,
   linkcolor={red!50!black},
   citecolor={green},
   urlcolor={blue!80!black}
}

\theoremstyle{plain}

\theoremstyle{definition}

\usepackage{tikz}
\usetikzlibrary{matrix,chains,positioning,decorations.pathreplacing,arrows}
\usetikzlibrary{shapes,arrows}
\tikzset{
	font={\fontsize{9pt}{12}\selectfont}}

\ExplSyntaxOn

\bool_new:N \g_noteobserve

\NewDocumentCommand{\setnote}{}{
  \bool_gset_true:N \g_noteobserve
}

\NewDocumentCommand{\setobserve}{}{
  \bool_gset_false:N \g_noteobserve
}

\NewDocumentCommand{\nobs}{ o }{
  \IfValueT{#1}{
    \str_if_eq:noTF {note} {#1} {
      \bool_gset_true:N \g_noteobserve
    } {
      \str_if_eq:noTF {Note} {#1} {
        \bool_gset_true:N \g_noteobserve
      } {
        \bool_gset_false:N \g_noteobserve
      }
    }
  }
  \bool_if:nTF { \g_noteobserve } {
    \bool_gset_false:N \g_noteobserve
    note
  } {
    \bool_gset_true:N \g_noteobserve
    observe
  }
  \IfValueF{#1}{~}
}

\NewDocumentCommand{\Nobs}{ o }{
  \IfValueT{#1}{
    \str_if_eq:noTF {note} {#1} {
      \bool_gset_true:N \g_noteobserve
    } {
      \str_if_eq:noTF {Note} {#1} {
        \bool_gset_true:N \g_noteobserve
      } {
        \bool_gset_false:N \g_noteobserve
      }
    }
  }
  \bool_if:nTF { \g_noteobserve } {
    \bool_gset_false:N \g_noteobserve
    Note
  } {
    \bool_gset_true:N \g_noteobserve
    Observe
  }
  \IfValueF{#1}{~}
}

\int_new:N \g_furthermore

\NewDocumentCommand{\Moreover}{ o o }{
  \IfValueT{#1}{
    \str_case:nn {#1} {
      {Furthermore} {\int_set:Nn {\g_furthermore} {0}}
      {Moreover} {\int_set:Nn {\g_furthermore} {1}}
      {In~addition} {\int_set:Nn {\g_furthermore} {2}}
      {note} {\bool_gset_true:N \g_noteobserve}
      {observe} {\bool_gset_false:N \g_noteobserve}
    }
    \IfValueT{#2}{
      \str_case:nn {#2} {
        {Furthermore} {\int_set:Nn {\g_furthermore} {0}}
        {Moreover} {\int_set:Nn {\g_furthermore} {1}}
        {In~addition} {\int_set:Nn {\g_furthermore} {2}}
        {note} {\bool_gset_true:N \g_noteobserve}
        {observe} {\bool_gset_false:N \g_noteobserve}
      }
    }
  }
  \int_case:nn { \int_mod:nn {\g_furthermore} {3} } {
    { 0 } { Furthermore,~\nobs that}
    { 1 } { Moreover,~\nobs that}
    { 2 } { In~addition,~\nobs that}
  }
  \int_incr:N \g_furthermore
  \IfValueF{#1}{~}
}

\bool_new:N \g_hencetherefore

\NewDocumentCommand{\hence}{}{
  \bool_if:nTF { \g_hencetherefore } {
    \bool_gset_false:N \g_hencetherefore
    hence~
  } {
    \bool_gset_true:N \g_hencetherefore
    therefore~
  }
}

\NewDocumentCommand{\Hence}{}{
  \bool_if:nTF { \g_hencetherefore } {
    \bool_gset_false:N \g_hencetherefore
    Hence,~we~obtain~
  } {
    \bool_gset_true:N \g_hencetherefore
    Therefore,~we~obtain~
  }
}

\seq_new:N \g_cflist_loaded
\seq_new:N \g_cflist_pending

\NewDocumentCommand{\cfadd}{ m }
{
	\seq_if_in:NnF \g_cflist_loaded { #1 } {
		\seq_if_in:NnF \g_cflist_pending { #1 } {
			\seq_gput_right:Nn \g_cflist_pending { #1 }
		}
	}
}

\NewDocumentCommand{\cfconsiderloaded}{ m }{
	\seq_gput_right:Nn \g_cflist_loaded {#1}
}

\NewDocumentCommand{\cfremove}{ m }
{
	\seq_gremove_all:Nn \g_cflist_pending { #1 }
}

\NewDocumentCommand{\cfload}{ o }
{
	\seq_if_empty:NTF \g_cflist_pending {\unskip} {
		(cf.\ \cref{\seq_use:Nn \g_cflist_pending {,}})\IfValueTF{#1}{#1~}{\unskip}
		\seq_gconcat:NNN \g_cflist_loaded \g_cflist_loaded \g_cflist_pending
		\seq_gclear:N \g_cflist_pending
	}
}

\NewDocumentCommand{\cfclear} {} {
	\seq_gclear:N \g_cflist_loaded
	\seq_gclear:N \g_cflist_pending
}

\NewDocumentCommand{\cfout}{ o }
{
	\seq_if_empty:NTF \g_cflist_pending {\unskip} {
		(cf.\ \cref{\seq_use:Nn \g_cflist_pending {,}})\IfValueTF{#1}{#1~}{\unskip}
		\seq_gclear:N \g_cflist_pending
	}
}

\NewDocumentCommand{\ifnocf} { m } {
	\seq_if_empty:NT \g_cflist_pending { #1 }
}

\ExplSyntaxOff

\NewDocumentEnvironment{cproof}{m}
{\begin{proof}[Proof of \cref{#1}]}%
	{\noindent The proof of \cref{#1} is thus complete.
\end{proof}}

\NewDocumentEnvironment{cproof2}{m}
{\begin{proof}[Proof of \cref{#1}]}%
	{\noindent This completes the proof of \cref{#1}.
\end{proof}}

\begin{document}

\title{Transformed $\ell_1$ Regularizations for Robust Principal Component Analysis: Toward a Fine-Grained Understanding}

\author[1]{
Kun Zhao
	\thanks{
{\tt kun.zhao@utdallas.edu}.
	}
}
\author[2]{
Haoke Zhang
        \thanks{
{\tt zhkzhk203@gmail.com}.
        } 
}
\author[1]{
Jiayi Wang
	\thanks{ Corresponding author.
		{\tt jiayi.wang2@utdallas.edu}.
	}
}
\author[2]{
Yifei Lou
        \thanks{
{\tt yflou@unc.edu}.
        } 
}

\affil[1]{Department of Mathematical Sciences, The University of Texas at Dallas, 800 W. Campbell Rd, Richardson, TX 75080, USA}
\affil[2]{Department of Mathematics and School of Data Sciences and Society, The University of North Carolina at Chapel Hill, Chapel Hill 27599, NC, USA}

\begin{abstract}
Robust Principal Component Analysis (RPCA) aims to recover a low-rank structure from noisy, partially observed data that is also corrupted by sparse, potentially large-magnitude outliers. Traditional RPCA models rely on convex relaxations, such as nuclear norm and $\ell_1$ norm, to approximate the rank of a matrix and the $\ell_0$ functional (the number of non-zero elements) of another. In this work, we advocate a nonconvex regularization method, referred to as transformed $\ell_1$ (TL1), to improve both approximations. The rationale is that by varying the internal parameter of TL1, its behavior asymptotically approaches either $\ell_0$ or $\ell_1$. Since the rank is equal to the number of non-zero singular values and the nuclear norm is defined as their sum, applying TL1 to the singular values can approximate either the rank or the nuclear norm, depending on its internal parameter. We conduct a fine-grained theoretical analysis of statistical convergence rates, measured in the Frobenius norm, for both the low-rank and sparse components under general sampling schemes. These rates are comparable to those of the classical RPCA model based on the nuclear norm and $\ell_1$ norm. Moreover, we establish constant-order upper bounds on the estimated rank of the low-rank component and the cardinality of the sparse component in the regime where TL1 behaves like $\ell_0$, assuming that the respective matrices are exactly low-rank and exactly sparse. Extensive numerical experiments on synthetic data and real-world applications demonstrate that the proposed approach achieves higher accuracy than the classic convex model, especially under non-uniform sampling schemes. 

\end{abstract}

%
%
%
\keywordone{Robust principal component analysis,}
\keywordtwo{Transformed $\ell_1$ regularization, }
\keywordthree{General sampling distribution, }
\keywordfour{Low-rankness, Sparsity,}
\keywordfive{Non-asymptotic upper bound.}

\maketitle

\section{Introduction}
In numerous scientific and engineering disciplines, high-dimensional datasets often exhibit underlying low-dimensional structures governed by a limited number of intrinsic factors. Principal Component Analysis (PCA) is a foundational technique for uncovering such low-dimensional representations, enabling dimensionality reduction and feature extraction \cite{jolliffe2016principal, kambhatla1997dimension,dehoop2022costaccuracy}. However, traditional PCA is notoriously sensitive to outliers and corruptions \cite{hubert2009robust, jolliffe2016principal}: even a small fraction of grossly corrupted entries can significantly distort the estimated principal components. This vulnerability poses a major challenge in real-world scenarios, where corruptions are frequently sparse but can be large in magnitude. Early efforts to improve PCA focused on modifying the estimation of the covariance or correlation matrices \cite{croux2000principal,hubert2005robpca,rousseeuw2003robust} to reduce sensitivity to anomalous data points. A major breakthrough came with its reformulation as a matrix decomposition problem, in which the observed data matrix $M_0$ is formulated as the sum of a low-rank matrix $L_0$, representing the underlying structure, and a sparse component $S_0$, capturing anomalies or corruptions. 
This approach, known as Robust Principal Component Analysis (RPCA)\cite{candes2011robust}, 
extends classical PCA to handle outliers and sparse corruptions. 
Owing to its ability to recover meaningful structures from contaminated data, RPCA has gained 
broad popularity in statistical machine learning and has found wide applications in areas such 
as video surveillance
\cite{bouwmans2014robust, liu2015background,rezaei2017background,rodriguez2016incremental,zheng2025tensor}, face recognition \cite{ahmadi2023detection, he2011robust, xue2019side, wang2021multi}, anomaly detection \cite{agrawal2016adaptive, genel2024application, jin2020anomaly}, and image denosing \cite{yuan2014efficient,chen2023unsupervised, tu2023new}. It has inspired numerous algorithmic developments for solving RPCA, including methods based on Augmented Lagrange Multipliers (ALM) \cite{lin2010augmented,bhardwaj2016robust,sha2020fast}, Accelerated Proximal Gradient
(APG) \cite{ganesh2009fast,toh2010accelerated}, and Alternating Direction Method of Multipliers (ADMM) \cite{wang2014parallel,yang2018alternating,gao2024low}. 

The perspective of RPCA, popularized in \cite{candes2011robust}, led to the Principal Component Pursuit (PCP) framework: a convex optimization problem that minimizes a combination of the nuclear norm and the $\ell_1$ norm of the respective components, with guaranteed exact recovery in a noiseless setting under strong incoherence conditions on the low-rank matrix $L_0$. Numerous variants of PCP have been proposed to address more complex real-world scenarios, including stable PCP for noisy observations \cite{aravkin2014variational,yin2019stable,zhou2010stable}, block-based PCP \cite{liu2016block}, and local PCP \cite{wohlberg2012local,zhang2015analysis} for structured corruptions, along with adaptations designed to handle missing data. However, much of the existing theoretical analysis relies on strong assumptions, such as incoherence conditions on the low-rank matrix, uniformly distributed observed entries, and uniformly located non-zero entries in the sparse component \cite{candes2011robust,wright2009robust,xu2010robust}, which may be too restrictive to be realistic. 
For example, in video surveillance, video frames are stacked as columns of a data matrix, with the goal of separating the static background (i.e., low-rank structure) from moving objects (i.e., sparse components). Foreground objects, such as people or vehicles, typically appear in contiguous regions within each frame, resulting in clustered or grouped patterns rather than randomly or uniformly scattered outliers, thereby violating the assumptions made in earlier works.

Since missing data is ubiquitous in real-world applications \cite{austin2021missing,popa2021improved,enders2022applied,woods2024best}, there is a growing interest in studying a non-uniform observed model where each entry is observed independently with varying probabilities \cite{cai2016matrix,fang2018max,li2024pairwise,wang2024robust} and no specific assumptions on the support of the sparse component \cite{chen2011robust,cherapanamjeri2017nearly}. Chen et al.~\cite{chen2011robust} investigated robust matrix completion with uniformly distributed observations and made no assumptions on the distribution of columnwise corruptions, while Cherapanamjeri et al.~\cite{cherapanamjeri2017nearly} focused on arbitrary entrywise corruptions under a uniform sampling regime. Klopp et al.~\cite{klopp2017robust} extended this line of work on matrix recovery to a general sampling distribution with incomplete observations and derived non-asymptotic upper bounds for estimation errors measured by the Frobenius norm. However, they used the nuclear norm as a convex relaxation to the rank to enforce the low-rank structure, which has been empirically reported to overestimate the true rank \cite{wang2021matrix,zhao2025noisy}.  

Inspired by recent advances in using the Transformed $\ell_1$ (TL1) penalty in sparse recovery \cite{zhang2018minimization} and in low-rank matrix completion \cite{zhao2025noisy}, we propose a novel RPCA model that adapts TL1 to both the low-rank and sparse components. Computationally, we design an efficient algorithm to solve the proposed model based on the Alternating Direction Method of Multipliers (ADMM) \cite{boydPCPE11admm}. Theoretically, we derive the non-asymptotic upper error bounds for the estimated low-rank and sparse components. Specifically, in the absence of corruption and with appropriately selected hyperparameters, we show that our approach attains the minimax optimal rate up to logarithmic factors. 
We relax the assumptions made in \cite{klopp2017robust} by allowing observations to arise from a general sampling distribution under milder conditions, and we do not impose any structural patterns or distributional assumptions on the corruptions. These relaxed assumptions make the model more realistic and enhance the practical feasibility of robust component separation.  
Furthermore,  we demonstrate the advantage of TL1 regularizations in controlling rank and sparsity estimations. Specifically, with appropriate choices of hyperparameters in the proposed model, both the estimated rank and sparsity level can achieve constant-order accuracy relative to the true rank and cardinality, respectively. Experimentally, we conduct a comprehensive simulation study under various missing data scenarios with corruptions. Additionally, we apply the proposed model to a synthetic video and a real video dataset to illustrate its effectiveness in practical settings. 
In summary, our main \textbf{contributions} are four-fold: 
\begin{enumerate}[leftmargin=*]
 \item \textbf{Numerical algorithm:} We design an efficient ADMM scheme to solve the proposed TL1-regularized RPCA model (see Section \ref{algorithm}).
      \item \textbf{Fine-grained error bound analysis:} Our non-asymptotic upper error bounds for both estimated low-rank and sparse components across different sampling schemes are compared with existing literature under relaxed assumptions.
      We further demonstrate that the minimax optimal rates can be achieved (see Section \ref{sec:error_bounds}). 
    \item \textbf{Sparsity and rank estimation:} 
    We show that TL1 regularization effectively controls rank and sparsity estimation, with appropriate hyperparameters yielding estimates that match the true rank and cardinality up to a constant factor (see Section \ref{sec:card_bounds}). 
    \item \textbf{Extensive experiments:} We validate the recovery performance through extensive simulations under various scenarios, and demonstrate the model’s practical effectiveness on both synthetic and real video datasets (see Section \ref{sec:experiments}).
\end{enumerate}


\section{Proposed approach}
\subsection{Problem setup}\label{sec:setup}
Suppose an underlying matrix $M_0 \in \mathbb{R}^{m_1 \times m_2}$ can be decomposed as $M_0 = L_0 + S_0,$ where $L_0$ has a low rank and $S_0$ is sparse (i.e., only having a few non-zero elements). 
Given $N$ independent noisy observations $(T_i, Y_i)$ that satisfy the trace regression model in \cite{klopp2017robust}:
\begin{align}\label{trace regression model}
    Y_i=\text{tr}\big(T_i^\intercal M_0\big) + \sigma\xi_i = \langle T_i, M_0\rangle + \sigma\xi_i, \quad \text{for} \ i = 1, \dots, N,
\end{align}
the goal of RPCA is to estimate the matrix $M_0,$ and more specifically, to identify its low-rank component $L_0$ and sparse counterpart $S_0.$ 
In \eqref{trace regression model}, each matrix $T_i \in \mathbb{R}^{m_1 \times m_2}$ is an i.i.d. copy of a random indicator matrix with distribution $\Pi=(\pi_{kl})_{k,l=1}^{m_1, m_2}$ over the set:
\begin{equation*}
    \Gamma=\{e_k(m_1)e_l^\top(m_2), k\in[m_1], l\in [m_2]\}, 
\end{equation*}
where $\pi_{kl}$ is the probability that a particular sample is located at position $(k,l)$, $e_k(m_j)$ represents the $k$-th canonical basis vector in $\mathbb{R}^{m_j}$, with a 1 in the $k$-th entry and zeros elsewhere, and $[m_j]=\{1,\dots,m_j\}$ for $j=1,2$.  The term $\sigma\xi_i$ in \eqref{trace regression model} represents the noise, where $\xi_i$ are i.i.d.~zero-mean random variables with variance 1 
and $\sigma\geq 0$ denotes the standard deviation. 

Among these $N$ observations, we assume that $n$ of them are not influenced by $S_0$ and are referred to as \textit{uncorrupted}; the remaining $N-n$ observations are affected by both $L_0$ and $S_0,$ corresponding to the observed (non-zero) entries in $S_0.$ Based on whether an observation is corrupted, we partition the indices of the observations into two disjoint sets: $\Omega$ and $\tilde{\Omega}$. 
The set $\Omega$ contains the indices of observed, uncorrupted entries from $L_0$, where the corresponding entries in $S_0$ are zero. The set $\tilde{\Omega}$ contains the indices of the observed non-zero entries in $S_0$.
We define the index sets: $\tilde{\mathcal{I}} \subseteq [m_1] \times [m_2]$ as the support of $S_0$ (i.e., the set of indices where $S_0$ is non-zero) and $\mathcal{I}$ is the complement of $\tilde{\mathcal{I}}$.  Then, we have $|\Omega| = n$, $|\tilde{\Omega}| = |\tilde{\mathcal{I}}| = N - n$. We set $\beta = N/n$. 


\textbf{Notations:}
We introduce the following notations, which will be used throughout this paper. For any index set $\mathcal{I}$, we denote its cardinality by $|\mathcal{I}|$. For a matrix $A \in \mathbb{R}^{m_1 \times m_2}$, let $m = \min(m_1, m_2) = (m_1 \wedge m_2)$, $M = \max(m_1, m_2) = (m_1 \vee m_2)$, $d = m_1 + m_2$. 
The trace of $A$ is denoted by $\text{tr}(A)$. 
Additionally, we define several matrix norms\footnote{Note that the $\ell_0$ ``norm,'' including $\|A\|_0,$ is not a norm in the strict mathematical sense.}: 
$\|A\|_\infty = \max_{k,l} |A_{kl}|$, $\|A\|_F = \sqrt{{\sum_{k,l}}A^2_{kl}}$, $\|A\|_1 = \sum_{k,l} |A_{kl}|$, $\|A\|_0 = \sum_{i,j} I\{A_{kl} \neq 0\},$
where $A_{kl}$ denotes the value of $(k,l)$-th entry of $A$ and $I\{\cdot\}$ is the indicator function. Denote $\sigma_j(A)$ as the $j$th singular values of $A$  in a descending order, then the nuclear norm is defined as $\|A\|_* = \sum_{j=1}^{m} \sigma_j(A)$ and the spectral norm $\|A\| = \sigma_1(A)$. Given the sampling distribution $\Pi$, 
 we define 
$L_2(\Pi)$ norm of $A$  by $\|A\|^2_{L_2(\Pi)}=\mathbb{E}(\langle A, T \rangle^2)=\sum_{k=1}^{m_1}\sum_{l=1}^{m_2} \pi_{kl}A^2_{kl}.$  
Finally, we introduce the following asymptotic notations for theoretical analysis. For any two non-negative sequences $\{a_n\}$ and $\{b_n\}$, we say $a_n = \mathcal{O}(b_n)$ if there exists a constant $C$ such that  $a_n \leq Cb_n$ and $a_n = \mathcal{O}_p(b_n)$ if there exists a constant $C'$ such that  $a_n \leq C'b_n$ with high probability; $a_n = {\scriptstyle \mathcal{O}}(b_n)$ if there is a constant $C''$ such that $a_n < C''b_n$. We denote $a_n \asymp b_n$ if $a_n = \mathcal{O}(b_n)$ and $b_n = \mathcal{O}(a_n)$.


\subsection{Problem formulation}
\label{sec:model}

We define the TL1 regularization on a matrix $A\in\mathbb R^{m_1\times m_2}$ and provide its asymptotical behaviors, \begin{equation}\label{eq:TL1onL}
     \Phi_{a}(A) = \sum_{j=1}^m \frac{(a+1)\sigma_j(A)}{a+\sigma_j(A)}, \  \text{ with } \    \underset{a \rightarrow 0+}{\lim} \Phi_{a}(A) = \rank(A), \; \underset{a \rightarrow \infty}{\lim} \Phi_{a}(A) = \|A\|_*,
\end{equation}
where $a>0$ is a hyperparameter. In addition, the TL1 regularization can be applied to each matrix element as an interpolation between the $\ell_0$ and $\ell_1$ norms, 
\begin{equation}\label{eq:TL1onS}
    \phi_{a}(A) = \sum_{i,j} \frac{(a+1)|A_{ij}|}{a+|A_{ij}|},\  \text{ with } \     \underset{a \rightarrow 0+}{\lim} \phi_{a}(A) = \|A\|_0, \; \underset{a \rightarrow \infty}{\lim} \phi_{a}(A) = \|A\|_1.
\end{equation}

The TL1 penalty has been studied in the context of low-rank matrix completion \cite{zhang2015transformed, zhao2025noisy} and sparse signal recovery \cite{zhang2018minimization}. However, the incorporation of TL1 penalties for both low-rank and sparse components in the RPCA framework has not been previously investigated.


We propose the use of TL1 regularization for recovering a low-rank component $\hat L$ and a sparse component $\hat S$ from observed independent pairs $(T_i, Y_i)$ for $i = 1,\dots, N$ 
:
\begin{align}\label{est:hat_LS}
    (\hat{L}, \hat{S}) = \underset{\|L\|_\infty \leq \zeta, \|S\|_\infty \leq \zeta}{\arg\min} \bigg\{ \frac{1}{N} \sum_{i=1}^{N} (Y_i - \langle T_i, L+S \rangle)^2 + \lambda_1 \Phi_{a_1} (L) + \lambda_2 \phi_{a_2} (S)\bigg\}, 
\end{align}
where $a_1,a_2$ are positive hyperparameters for $\Phi_{a_1}(\cdot),\phi_{a_2}(\cdot),$ respectively, $\lambda_1,\lambda_2$ are positive weighting parameters, and $\zeta>0$ is regarded as an upper bound on the entrywise magnitude of both estimators. In practice, $\zeta$ serves as a form of prior knowledge. For example, when separating a video frame into background and moving objects, the matrix values corresponding to pixel intensities typically fall within $[0, 1]$, where we can set $\zeta = 1$. Owing to the properties in \eqref{eq:TL1onL} and \eqref{eq:TL1onS}, the TL1 functions
$\Phi_{a_1}(\cdot)$ and $\phi_{a_2}(\cdot)$  can effectively promote low-rank and sparse structures  by appropriately tuning the parameters $a_1$ and $a_2$, respectively; please refer to our theoretical analysis in Section \ref{sec:card_bounds}. 

\input{admm}
Note that a good initial value is important for finding a desirable pair of matrices when minimizing \eqref{eq:obj}. We choose the classic convex model \cite{candes2011robust} that combines the nuclear norm and the $\ell_1$ norm, referred to as L1 for conciseness, and use its output as the initial value for minimizing the TL1-regularized model \eqref{est:hat_LS}. We implement the L1 model by replacing the proximal operator $\text{prox}^{\text{TL1}}_{a}$ defined in \eqref{eq:prox-TL1} by the soft shrinkage operator \cite{boydPCPE11admm}. Define the objective function in \eqref{est:hat_LS} by 
\begin{equation}\label{eq:obj}
   Q(L,S):=\frac{1}{N} \sum_{i=1}^{N} (Y_i - \langle T_i, L+S \rangle)^2 + \lambda_1 \Phi_{a_1} (L) + \lambda_2 \phi_{a_2} (S).
\end{equation}
The stopping criteria for the TL1 method are $|(f_{k+1} - f_k)/f_k| < 10^{-3}$ with $f_k=Q(L^k,S^k)$ and a maximum of 1000 iterations. 

We summarize the ADMM scheme for minimizing \eqref{est:hat_LS} in Algorithm \ref{alg:TL1-TL1}. Note that the main computational cost of Algorithm \ref{alg:TL1-TL1} lies in the SVD, which is $O(mm_1m_2)$ and is the same as that of the L1 approach.



\begin{algorithm}
\caption{TL1-regularized  RPCA via ADMM} \label{alg:TL1-TL1}
\begin{algorithmic}[1]
\State Input: $Y \in \mathbb{R}^{m_1 \times m_2}$, $T \in \mathbb{R}^{m_1 \times m_2}$
\State Set parameters: $a_1$,$a_2, \lambda_1, \lambda_2, \zeta, \rho_1, \rho_2 \in (0,\infty)$ 
\State Initialize $(J^0,R^0)$ is obtained by the L1 model, $B^0 =D^0 =\mathbf{0}_{m_1 \times m_2}$, $k = 0$
\While{stopping criteria not satisfied}
    \State $L^{k+1} \gets U\mbox{diag}\left(\{\text{prox}^{\text{TL1}}_{a_1}(\sigma_k, \lambda_1/\rho_1)\}_{1\leq k\leq m} \right)V^\intercal$ with $J^k-B^k =  U \text{diag}(\{\sigma_k\}_k) V^T$.
    
    \State $S^{k+1} \gets \text{prox}^{\text{TL1}}_{a_2} \left( R^k-D^k, \lambda_2/\rho_2 \right).$
    
    \State $J^{k+1} \gets  \min \left\{ \max\Big\{ \left( \frac{2}{N} T \circ (Y-R^k) + \rho_1 (L^{k+1}+B^k) \right) \oslash \left( \frac{2}{N} {T} + \rho_1 I_d\right), \zeta \Big\}, -\zeta \right\}.$
    \State $R^{k+1} \gets \min \left\{ \max\Big\{ \big( \frac{2}{N} T \circ (Y-J^{k+1}) + \rho_2 (S^{k+1}+D^k) \big) \oslash \big( \frac{2}{N} {T} + \rho_2 I_d\big), \zeta \Big\}, -\zeta \right\}.$
    \State $B^{k+1} \gets B^k +  (L^{k+1} - J^{k+1}).$
    \State $D^{k+1} \gets B^k +  (S^{k+1} - R^{k+1}).$
    \State $k \gets k + 1.$
\EndWhile

\State Output: $\hat L = L^k, \hat S = S^k$
\end{algorithmic}
\end{algorithm}

\begin{remark}
    Introducing two auxiliary variables to decouple the problem into low-rank and sparse components yields a multi-block ADMM formulation, which is known to potentially lack convergence guarantees (see, e.g., \cite{chen2016direct}). Accordingly, we do not pursue a convergence proof and instead focus on establishing theoretical error bounds. While a modified ADMM scheme with convergence guarantees (e.g., \cite{chen2017efficient}) could be considered, prior studies have reported that such variants are often less efficient in practice. A systematic study of this tradeoff between convergence guarantees and empirical performance is left for future work.
\end{remark}

\section{Theoretical properties}

In this section, we investigate the theoretical properties of the estimators $\hat{L}$ and $\hat{S}$ obtained from the proposed model \eqref{est:hat_LS}.  We begin with assumptions about the sampling scheme and noise distribution. 
Recall the definitions of $\mathcal{I}$ and $\tilde {\mathcal{I}}$  in Section \ref{sec:setup}, we divide the matrices $T_i$ into two sets accordingly, 
\[
\Gamma'=\{e_k(m_1)e_l^\top(m_2),\; (k,l)\in \tilde{\mathcal{I}}\} \ \text{and} \  \Gamma''=\{e_k(m_1)e_l^\top(m_2),\;(k,l)\in\mathcal{I}\}.
\]

Unlike some of the existing literature \cite{candes2011robust,xu2010robust,zhou2010stable,zhou2011godec} that assumes corrupted entries follow a uniform sampling distribution, we follow the works \cite{chen2011robust, chen2015matrix, klopp2014noisy} to avoid imposing strict distributional assumptions on the support set $\tilde{\mathcal I}$.  
This choice is also motivated by practical considerations: in real-world applications such as background subtraction and anomaly detection, corruptions often exhibit spatial and temporal structure rather than uniform randomness. For instance, in our experiment (see Section \ref{sec:realdata}), moving objects tend to appear in specific regions, rendering the uniform corruption assumption unrealistic. 
Since the unobserved entries of $S_0$ are not identifiable, we restrict estimation to the support $\tilde{\mathcal{I}}$, effectively treating all unobserved entries of $S_0$ are zero; see, e.g., \cite{chen2011robust,klopp2017robust}. 


We impose mild assumptions about the sampling distribution on the set $\mathcal I$, which are commonly used in the literature \cite{klopp2014noisy, klopp2017robust, zhao2025noisy}. Define $C_l = \sum_{k=1}^{m_1} \pi_{kl}$ and $R_k = \sum_{l=1}^{m_2} \pi_{kl}$ as the probabilities that an observation appears in the $l$-th column and the $k$-th row, respectively,
for $k\in [m_1]$ and $l\in[m_2]$. 

By the definitions along with the constraints $\sum_{l=1}^{m_2} C_l = 1$ and $\sum_{k=1}^{m_1} R_k = 1$, we have $\underset{l}{\max}\; C_l \geq 1/m_2$ and $\underset{k}{\max} \; R_k \geq 1/m_1$,  implying that $\underset{k,l}{\max}(R_k, C_l) \geq 1/m$. 
\begin{assumption}\label{assump1}
   There exists a constant $G \geq 1$ such that for any $(k,l) \in \mathcal{I},\; \max_{k,l}(R_k, C_l) \leq G/m$.
\end{assumption}

\begin{assumption}\label{assump2}
    There exists a constant $\nu \geq 1$ such that $1/(\nu |\mathcal{I}|) \leq \pi_{kl} \leq \nu/|\mathcal{I}|.$
\end{assumption}

\begin{assumption}\label{assump3}
    There exists a positive constant $c_1$ such that $\max_{i\in \Omega}\mathbb{E} [\exp (|\xi_i| / c_1)] \leq e$, where $\xi_i$ are sub-exponential noise variables  
    and 
    $e$ is the base of the natural logarithm. 
\end{assumption}

Assumption \ref{assump1} ensures that no individual row or column in the index set $\mathcal{I}$
is sampled with 
high probability and a larger value of $G$
reflects a greater imbalance in the sampling distribution, resulting in a more non-uniform sampling scheme over the uncorrupted regions of the low-rank component. Assumption \ref{assump2} implies that $$1/(\nu |\mathcal{I}|) \|A_{\mathcal{I}}\|_F^2 \leq \|A_{\mathcal{I}}\|_{L_2(\Pi)}^2 \leq \nu/|\mathcal{I}| \|A_{\mathcal{I}}\|_F^2.$$ For a uniform sampling distribution, both $G$ and $\nu$ are taken as 1. Assumption \ref{assump3} is a mild assumption on the noise.
It is worth noting that when the underlying matrix is fully observed, i.e., all entries are available without missingness, Assumptions~\ref{assump1}--\ref{assump3} are not required.

In what follows, Section \ref{sec:error_bounds} establishes upper bounds on the estimation errors for $\hat{L}$ and $\hat{S}$. 
 In Section \ref{sec:card_bounds}, we study the low-rankness and sparsity of $\hat{L}$ and $\hat{S},$ controlled within constant orders by varying the parameters $a_1$ and $a_2,$ respectively. Please refer to Appendix \ref{append:theories} for proof details.

\subsection{Error bound analysis} \label{sec:error_bounds}

We first present in Theorem \ref{thm:gen_upp_S} an error bound for the estimated sparse component, assuming the true matrix $S_0$ is exactly sparse.
\begin{theorem}\label{thm:gen_upp_S}
    Suppose Assumptions \ref{assump1} - \ref{assump3} hold,   $S_0 \in \mathbb{R}^{m_1 \times m_2}$ is exactly sparse, i.e., $\|S_0\|_0 \leq s_0$ for a small integer $s_0$, and $\|S_0\|_\infty \leq \zeta$ with the same constant $\zeta$ in \eqref{est:hat_LS}. 
    Take $\lambda_2^{-1} = \mathcal{O}(\{(\sigma \vee \zeta)\log d/N\}^{-1})$, then for any $a_2 >0$, there exit two positive constants $C_1$ and $C_2$ depending on $c_1$ such that the estimator $\hat S$ from \eqref{est:hat_LS}   satisfies
    \begin{align}\label{general_bound_S}
        \frac{\|\hat{S} - S_0\|_F^2}{m_1m_2} \leq C_1 s_0(\sigma\vee\zeta)^2 \frac{\log d}{m_1m_2} + C_2 \min\left\{ \frac{\lambda_2N}{m_1m_2}\phi_{a_2}(S_0),  \frac{s_0\lambda_2^2N^2 }{m_1m_2} \left( \frac{a_2+1}{a_2}\right)^2  \right\} ,
    \end{align}
    with probability at least $1-2/d$.
    Moreover, if we take 
    $\lambda_2 \asymp (\sigma \vee \zeta)(\log d)/{N}$, then for any $a_2 >0$, there exists a positive constant $C_3$ depending on $c_1$ such that the following inequality,
    \begin{align}\label{exact_bound_s}
         \frac{\|\hat{S} - S_0\|_F^2}{m_1m_2} \leq C_3 s_0(\sigma\vee\zeta)^2 \frac{\log d}{m_1m_2},
    \end{align}
    holds with probability at least $1-2/d$.
\end{theorem}
Note that the upper bound on $\hat{S}$ in \eqref{general_bound_S} 
remains unaffected by hyperparameters $\lambda_1$, $a_1$, and the structure of $L_0$. By choosing $\lambda_2\asymp (\sigma \vee \zeta){\log d}/{N},$ the bound \eqref{general_bound_S} reduces to \eqref{exact_bound_s}, which provides a non-asymptotic recovery guarantee  of order $\mathcal{O}\left(\log d/ (m_1m_2)\right)$ for sparse recovery. 

Next, we derive an upper error bound for the estimated low-rank matrix.
For convenience, 
we define 
$$\Delta_{S_0}(N,m_1,m_2):=C_1 s_0(\sigma\vee\zeta)^2{\log d}/{N} + C_2\min\left\{\lambda_2\phi_{a_2}(S_0), Ns_0\lambda_2^2\left({a_2+1}\right)^2/a_2^2\right\}.$$


\begin{theorem}\label{thm:gen_upp_L}
 Under the same assumptions in Theorem \ref{thm:gen_upp_S}, we further assume
    $L_0 \in \mathbb{R}^{m_1 \times m_2}$ satisfies $\|L_0\|_\infty \leq \zeta$. Take $\lambda_1^{-1} = \mathcal{O} ( \{ 
   [(\sigma \vee \zeta) (a_1 + \zeta\sqrt{m_1m_2}) \sqrt{{Gd\log d}}]/[(a_1+1)\sqrt{m_1m_2n}]
    \}^{-1} )$, then for any $n \gtrsim d\log d$ and $a_1>0$, there exist constants $C_4$, $C_5, C_6>0$ depending on $c_1$ such that
    the estimator $\hat{L}$ from \eqref{est:hat_LS} satisfies
    \begin{align}
       & \frac{\|\hat{L} - L_0\|_F^2}{m_1m_2} \leq 
         C_4\nu \beta \Delta_{S_0}(N,m_1,m_2) + \frac{4\zeta^2 s_0}{m_1m_2} \label{gen_upper_bound_L_2}\\
        & \qquad \qquad  + C_5\nu \beta\min\left\{(
         \lambda_1 \Phi_{a_1}(L_0), \beta \rank(L_0)\lambda_1^2 {(a_1+1)^2m_1m_2}/{a_1^2} \right\}  
         + C_6\nu\zeta^2 \sqrt{\frac{\log d}{n}} , \label{gen_upper_bound_L_1} 
    \end{align}
    with probability at least $1-(\kappa + 3)/d$ for a universal constant  $\kappa$.
\end{theorem}
The general upper bound in Theorem \ref{thm:gen_upp_L} has two components. The first component includes the two terms in \eqref{gen_upper_bound_L_2}, which arise from the presence of corruption. It vanishes in the absence of corruption, i.e., $s_0 = 0,$ 
leading to an upper bound consistent with the standard matrix completion setting, as discussed in  \cite{zhao2025noisy}. The second component \eqref{gen_upper_bound_L_1} 
mainly comes from the matrix completion error. 
Compared to \cite{zhao2025noisy}, our bound in \eqref{gen_upper_bound_L_1} is more general, accommodating any choice of
$\lambda_1, a_1>0$ that satisfies the conditions in Theorem \ref{thm:gen_upp_L}.
Moreover, the bound in \eqref{gen_upper_bound_L_1} can be further tightened under certain scenarios, as elaborated in the discussion following Corollary \ref{coro:exact_lowrank}.
In the presence of corruption, the order of $\Delta_{S_0}$ will not exceed that of \eqref{gen_upper_bound_L_1} if $\lambda_2$ is not too large.
For instance,  by choosing $\lambda_2 \asymp  (\sigma\vee\zeta){\log d}/{N}$, \eqref{gen_upper_bound_L_1} becomes dominant as
the terms in \eqref{gen_upper_bound_L_2} are bounded by 
$\mathcal{O}(\log d/N)$. 


Theorem~\ref{thm:gen_upp_L} indicates a smaller value of $\lambda_1$ corresponds to a tighter bound for a fixed value of $a_1$, and hence 
$\lambda_1 \asymp [(\sigma \vee \zeta) (a_1 + \zeta\sqrt{m_1m_2}) \sqrt{{Gd\log d}}]/[(a_1+1)\sqrt{m_1m_2n}]$ leads to the tightest error bound.  Using this choice of $\lambda_1$, we explore two specific scenarios: 
$L_0$ is approximately low-rank (Corollary \ref{coro:approx_large_a}) or exactly low-rank (Corollary \ref{coro:exact_lowrank}), while varying the parameter $a_1$.


\begin{corollary}\label{coro:approx_large_a}
    Under the same assumptions in Theorem \ref{thm:gen_upp_L}, we further assume $L_0 \in \mathbb{R}^{m_1 \times m_2}$ is approximately low-rank, i.e., $\|L_0\|_*/\sqrt{m_1m_2} \leq \gamma$ for a positive constant $\gamma$. Take 
    $\lambda_1 \asymp [(\sigma \vee \zeta) (a_1 + \zeta\sqrt{m_1m_2}) \sqrt{{Gd\log d}}]/[(a_1+1)\sqrt{m_1m_2n}],$
   then for any $n \gtrsim d\log d$, when $a_1^{-1} = \mathcal{O}((\sqrt{m_1m_2})^{-1})$, with probability at least $1-(\kappa + 3)/d$, the estimator $\hat{L}$ from \eqref{est:hat_LS} satisfies
    \begin{align*}
        \frac{\|\hat{L} - L_0\|_F^2}{m_1m_2} \leq C_4\nu\beta(\sigma\vee \zeta) \gamma\sqrt{\frac{Gd\log d}{n}} +  C_5\nu\zeta^2 \sqrt{\frac{\log d}{n}} + C_6\nu \beta \Delta_{S_0}(N,m_1,m_2) + \frac{4\zeta^2 s_0}{m_1m_2} .
    \end{align*}
\end{corollary}
This bound matches the bound derived in \cite[Theorem 1]{zhao2025noisy}, which is comparable to established results \cite{klopp2014noisy} for approximately low-rank matrices when $s_0=0.$
It also attains the minimax lower bound up to logarithmic factors 
\cite{negahban2012restricted}. 

\begin{corollary}\label{coro:exact_lowrank}
    Under the same assumptions in Theorem \ref{thm:gen_upp_L}, we further assume the rank of $L_0 \in \mathbb{R}^{m_1 \times m_2}$ is at most $r_0$ for an integer constant $r_0$. Take 
    $$\lambda_1 \asymp [(\sigma \vee \zeta) (a_1 + \zeta\sqrt{m_1m_2}) \sqrt{{Gd\log d}}]/[(a_1+1)\sqrt{m_1m_2n}],$$
     then for any $n \gtrsim d\log d$, with probability at least $1-(\kappa + 3)/d$, we have  the following three scenarios 
    \begin{enumerate}[leftmargin=*]
    \item[(i)] When $a_1^{-1} = \mathcal{O}(\left(\sqrt{m_1m_2}\right)^{-1})$,
    \begin{align}
        \label{eqn:best_error_bound}
            \frac{\|\hat{L} - L_0\|_F^2}{m_1m_2} \leq C_4\nu\beta^2(\sigma\vee \zeta)^2 r_0\frac{Gd\log d}{n} 
            + \Upsilon(n,m_1,m_2);
        \end{align}
        \item[(ii)] When $a_1^{-1} = \mathcal{O}((\sqrt{m_1m_2}\left(d\log d / n\right)^{1/4})^{-1})$ and $a_1 = \mathcal{O}\left(\left(\sqrt{m_1m_2}\right)\right)$, 
        \begin{align}
         \label{eqn:slow-error-bound-2}
            \frac{\|\hat{L} - L_0\|_F^2}{m_1m_2} \leq C_4\nu\beta^2(\sigma\vee \zeta)^2 (\frac{a_1+\zeta\sqrt{m_1m_2}}{a_1})^2 r_0 \frac{Gd\log d}{n} +  \Upsilon(n,m_1,m_2);
        \end{align}
        \item[(iii)] When $a_1 = \mathcal{O}(\sqrt{m_1m_2}(d\log d / n)^{1/4})$, 
        \begin{align}
        \label{eqn:slow-error-bound}
            \frac{\|\hat{L} - L_0\|_F^2}{m_1m_2} \leq C_4\nu\beta(\sigma\vee \zeta) r_0\sqrt{\frac{Gd\log d}{n}} +  \Upsilon(n,m_1,m_2), 
        \end{align}
    \end{enumerate}
    where   $\Upsilon(n,m_1,m_2) :=  C_5\nu\zeta^2 \sqrt{{\log d}/{n}} + C_6\nu \beta \Delta_{S_0}(N,m_1,m_2) + {4\zeta^2 s_0}/{(m_1m_2)}$.
\end{corollary}

When $a_1$ is sufficiently large, the bound \eqref{eqn:best_error_bound} in Scenario (i) is the same as the one in  
\cite[Theorem 2]{zhao2025noisy}, which is on par with the rate shown in \cite{negahban2012restricted} without corruption. 
Under this regime, if we choose $\lambda_2$ appropriately (e.g., $\lambda_2 \asymp (\sigma\vee\zeta){\log d}/{N}$), then $\|\hat{L} - L_0\|_F^2/(m_1m_2)$ + $\|\hat{S} - S_0\|_F^2/(m_1m_2)$  achieves the minimax optimal rate up to logarithm factors, as shown in \cite[Theorem 3]{klopp2017robust}.
In contrast, for a sufficiently small $a_1$ in  Scenario (iii), the bound \eqref{eqn:slow-error-bound} is dominated by the first term, yielding a slower convergence rate ($\sqrt{d\log d / n}$) compared to the order of ${d\log d / n}$ in Scenario (i). 
When $a_1$ falls into Scenario (ii), the convergence rate  \eqref{eqn:slow-error-bound-2} exhibits a  \textit{faster} order than $\sqrt{d\log d / n}$, improving the bound  in
\cite[Corollary 1]{zhao2025noisy}.

\begin{remark}
    In \cite{candes2011robust}, exact recovery of $L_0$ and $S_0$ is established in the noise-free setting (i.e., $\sigma=0$) with nuclear norm and $\ell_1$ regularizations. 
    Admittedly, there is a gap exhibited in our bounds (Theorems \ref{thm:gen_upp_S}-\ref{thm:gen_upp_L}) under mild assumptions when  $\sigma=0,$ due to non-ignorable components of $s_0 \log d/(m_1m_2)$ and $\sqrt{\log d/n}$. Nevertheless, the strict assumptions required in \cite{candes2011robust}, such as uniformly distributed corruptions and incoherence condition on $L_0,$  are significantly stronger than ours and may be violated in practice scenarios.  
\end{remark}

\subsection{Sparsity and low-rankness} \label{sec:card_bounds}
Compared to the nuclear norm and $\ell_1$ norm, which tend to overestimate rank and sparsity, TL1 regularizations offer more accurate approximations to the rank and the $\ell_0$ norm when the respective hyperparameter is sufficiently small. Here, we theoretically quantify
how the TL1 regularizations control the sparsity of the estimated sparse component (Theorem \ref{thm:sparsity}) and the rank of the recovered low-rank matrix (Theorem \ref{thm:lowrankness}).
For the ease of notation, we define 
\begin{align*}
 \Delta_{L_0}(n, m_1, m_2): =   C_5\nu \beta\min\left\{
         \lambda_1 \Phi_{a_1}(L_0), \beta \rank(L_0) \lambda_1^2 {(a_1+1)^2m_1m_2}/{a_1^2} \right\}.  
\end{align*}



\begin{theorem}[Sparsity]\label{thm:sparsity}
   Under the same assumptions in Theorem \ref{thm:gen_upp_L}, we take $\lambda_2^{-1} = \mathcal{O}(\{(\sigma \vee \zeta){\log d}/{N}\}^{-1})$ and $a_2 = {\scriptstyle \mathcal{O}}\left(\lambda_2 \{ \Delta_{L_0}(n,m_1,m_2) + \Delta_{S_0}(N, m_1, m_2) \}^{-1/2}\right)$, 
   then for any $n \gtrsim d\log d$, there exists a constant $C_7>0$ depending on $c_1$ such that  with probability at least $1-(\kappa + 3)/d$, we have
    \begin{multline}
        \|\hat{S}\|_0 \leq s_0 +  
        \frac{C_7}{a_2+1} 
        \max\Bigg\{ \lambda_2^{-1} \sqrt{s_0} 
        \sqrt{N\Delta_{S_0}(N, m_1, m_2)}
        \frac{\log d}{N} + \phi_{a_2}(S_0), \\
       \frac{1}{a_2+1} \frac{\log^2 d}{N} \Delta_{S_0}(N, m_1, m_2) \lambda_2^{-2} 
        + \{\Delta_{S_0}(N,m_1,m_2) + \Delta_{L_0}(n, m_1, m_2) \} \lambda_2^{-1} \Bigg\}. 
    \end{multline}
    If we further take $\lambda_2 \asymp \{\Delta_{S_0}(N,m_1,m_2) + \Delta_{L_0}(n, m_1, m_2) \}$, then $\|\hat{S}\|_0 = \mathcal{O}_p(s_0)$.
\end{theorem}

Theorem \ref{thm:sparsity}  reveals that for a sufficiently small $a_2,$ the number of non-zero entries decreases as $\lambda_2$ increases under a wide range of hyperparameter choices.   
We then present the bounds corresponding to explicit choices of hyperparameters
in Corollary \ref{cor:sparsity}, which yields a fast overall convergence rate while ensuring proper control over the cardinality of $\hat{S}$. 

\begin{corollary}
    \label{cor:sparsity}
    Under the same assumptions in Theorem  \ref{thm:sparsity}, we take $a_1^{-1} = \mathcal{O}((\sqrt{m_1m_2})^{-1})$,  
    $\lambda_1 \asymp [(\sigma \vee \zeta) \sqrt{{Gd\log d}}]/[\sqrt{m_1m_2n}]$, $\lambda_2 \asymp  (\sigma\vee\zeta){d\log d}/{N}$, 
    $a_2 = {\scriptstyle \mathcal{O}}\left(\sqrt{{d\log d}/{n}}\right)$, then we have 
    \begin{align*}
        \|\hat{S}\|_0 = \mathcal{O}_p(s_0) \quad\text{and} \quad
        \frac{\|\hat{L} - L_0\|_F^2}{m_1m_2}+ \frac{\|\hat{S} - S_0\|_F^2}{m_1m_2} = \mathcal{O}_p\left(r_0\frac{d\log d}{n} + s_0\frac{\log d}{m_1m_2}\right).
    \end{align*}
\end{corollary}

\begin{theorem}[Low-rankness] \label{thm:lowrankness}
    Under the same assumptions in Corollary \ref{coro:exact_lowrank}, we take $\lambda_1^{-1} = \mathcal{O} ( \{ 
    [(\sigma \vee \zeta) (a_1 + \zeta\sqrt{m_1m_2}) \sqrt{{Gd\log d}}]/[(a_1+1)\sqrt{m_1m_2n}]
    \}^{-1} )$, then for any $n \gtrsim d\log d$, when
   $a_1 = {\scriptstyle \mathcal{O}} \left((a_1+1)\lambda_1\left( \{ \Delta_{L_0}(n,m_1,m_2) + \Delta_{S_0}(N, m_1, m_2) \}^2Gd\log d/( nm_1m_2)\right)^{-1/4}\right)$, 
there exits  $C_8 > 0$ depending on $c_1$ such that  with probability at least $1-(\kappa + 3)/d$, we have
    \begin{multline}
         \mathrm{rank}(\hat{L}) \leq 
         \frac{C_8}{a_1+1}
         \max\Bigg\{ \lambda_1^{-1}
         \sqrt{\frac{Gd\log d}{N} \Delta_{L_0}(n, m_1, m_2) } 
         \sqrt{r_0} + \Phi_{a_1}(L_0), \\
         \lambda_1^{-1} 
         \{ \Delta_{L_0}(n,m_1,m_2) + \Delta_{S_0}(N, m_1, m_2) \}+ 
        \frac{1}{a_1+1}\lambda_1^{-2}\Delta_{L_0}(n, m_1, m_2)\frac{Gd\log d}{N}\Bigg\}.
        \notag
    \end{multline}
    

\end{theorem}

Theorem \ref{thm:lowrankness}  generalizes the bound obtained in  \cite[Theorem 4]{zhao2025noisy} under the corruption-free setting for general choices of hyperparameters. When combined with the error bound in Theory \ref{thm:gen_upp_L} for the low-rank matrix, it reveals a trade-off between estimation accuracy and the rank of the estimated low-rank component: increasing $\lambda_1$ yields a lower estimated rank but incurs a higher estimation error for a fixed value of $a_1$.  

Moreover, as shown in Corollary~\ref{cor:LR}, with appropriately selected hyperparameters, the rank of 
$\hat{L}$ can indeed be effectively controlled, though this comes at the cost of a slower convergence rate for the estimation error.

\begin{corollary}
    \label{cor:LR}
    Under the same assumptions in Theorem \ref{thm:lowrankness}, we take
  $ a_1 ={\scriptstyle {\mathcal{O}} }\left((m_1 m_2)^{1/4}\right)$, 
  $\lambda_1 \asymp [(\sigma \vee \zeta) (a_1 + \zeta\sqrt{m_1m_2}) \sqrt{{Gd\log d}}]/[(a_1+1)\sqrt{m_1m_2n}] $,
 $\lambda_2 \asymp  (\sigma\vee\zeta)\sqrt{{d\log d}/{N}}$, $a_2 = {\scriptstyle \mathcal{O}}\left(\{{d\log d}/{n}\}^{1/4}\right)$, then we have 
    \begin{align*}
     \|\hat{S}\|_0 = \mathcal{O}_p(s_0), \ \mathrm{rank}(\hat{L}) = \mathcal{O}_p(r_0), \   \frac{\|\hat{L} - L_0\|_F^2}{m_1m_2}+ \frac{\|\hat{S} - S_0\|_F^2}{m_1m_2} = \mathcal{O}_p\left(r_0\sqrt{\frac{d\log d}{n}} + s_0\frac{\log d}{m_1m_2}\right).
    \end{align*}
\end{corollary}



\begin{remark}
   Corollary \ref{cor:LR} has a limitation: although it provides bounds on the cardinality and rank of the estimated components, these bounds are only of constant order and do not guarantee exact recovery of the true sparsity or rank. Consequently, the conclusions do not reflect the ideal scenario where the estimated support or rank exactly matches the true underlying structure. Investigating whether the oracle property can be achieved (i.e., exact recovery of the true sparsity pattern or rank under suitable conditions) remains an important direction for future research. 
\end{remark}

\section{Experimental results}\label{sec:experiments}
We conduct extensive experiments to demonstrate the performance of the proposed TL1 approach in comparison to  the convex L1 approach \cite{candes2011robust}.
Quantitatively, we evaluate the performance in terms of relative error (RE) and Dice's coefficient (DC), defined as
\begin{equation*}
    \RE(\hat L,L_0)= \frac{\|\hat{L} - L_0\|_F}{\|L_0\|_F} \quad \text{and} \quad \DC(\hat S, S_0)=\frac{2|\supp(\hat S) \cap \supp(S_0)|}{|\supp(\hat S)| + |\supp(S_0)|},
\end{equation*}
where $(\hat{L},\hat S)$ is a pair of estimators of the ground truth $(L_0,  S_0)$ and $\supp(A)$ indicates the support set of the matrix $A$. 
Note that a higher DC score indicates a better recovery of the sparse structure.  
All the experiments are run on a Matebook 16 with an Intel i7-12700H chip and 16GB of memory, and the code implementation is publicly available on our GitHub: \href{https://github.com/zhanghaoke/Transformed-L1-Regularizations-for-Robust-Principal-Component-Analysis.git}{Transformed L1 Regularizations for RPCA}.

After discussing the experimental setup and parameter tuning in Section \ref{appendix:param}, we present simulation results in Section \ref{sect:simu} and video background separation in Section \ref{sec:realdata}.

\subsection{Experimental setup and parameter tuning}
\label{appendix:param}

    


We elaborate on two different sampling schemes for our \textit{simulated data}. Specifically, for each $(k,l) \in [m_1] \times [m_2]$, we define $\pi_{kl}$ as follows 
\begin{enumerate}[leftmargin=*]
    \item Uniform setting: $\pi_{kl} = 1/(m_1m_2)$.
    \item Non-uniform setting: $\pi_{kl} = p_kp_l$,
    where $p_k$ (or $p_l$) satisfies: 
    \begin{equation*}
         p_k = 
            \begin{cases}
            2p_0, & \text{if } k \leq \frac{m_1}{10}, \\
            4p_0, & \text{if } \frac{m_1}{10} < k \leq \frac{m_1}{5}, \\
            p_0, & \text{otherwise},
            \end{cases}
    \end{equation*}
    where $p_0$ is a normalized constant such that $\sum_{k=1}^{m_1} p_k = 1$. 
\end{enumerate}
   Then, for each entry in the matrix, we multiply $p_kp_l$ by a random number $r_{kl} \in [0, 1]$ to generate a score matrix $S_{kl} = p_kp_l \cdot r_{kl}$. We select the top entries in $S$ according to the given sampling ratio.

For parameter tuning, we begin with a candidate set of values for the respective hyperparameters for both TL1 and L1, as follows:
\begin{itemize}
    \item For TL1,
    \begin{itemize}
        \item $ a_1, a_2 \in \{10^{-2}, 5 \times 10^{-2}, 10^{-1}, 1, 10, 10^2\} $
        \item $ \lambda_1, \lambda_2 \in \{10^{-1},10^{-2},10^{-3},10^{-4}, 10^{-5},10^{-6},10^{-7},10^{-8},10^{-9}\} $
    \end{itemize}
    \item For L1,
    \begin{itemize}
        \item $ \lambda_1, \lambda_2 \in \{10^{-1},10^{-2},10^{-3},10^{-4}, 10^{-5},10^{-6},10^{-7},10^{-8},10^{-9}\} $
    \end{itemize}
\end{itemize}

For the \textit{simulation and synthetic video experiments}, where the ground-truth is available, we perform a grid search to find the optimal parameters that yield the minimum relative error (RE) between the recovered low-rank matrix and the ground-truth. 

For the \textit{real video datasets} without ground truth, we select a combination of parameters that (i) yields the smallest nonzero rank of $\hat{L}$, (ii) ensures the sparsity of $\hat{S}$ is below 40\%, and (iii) achieves a relative reconstruction error (i.e., $\|Y-\hat{L}-\hat{S}\|_F / \|Y\|_F$) of less than 1\%. This selection criterion is applied consistently to both L1 and TL1 models.

Following Algorithm~\ref{alg:TL1-TL1}, we implement the TL1 model by ourselves and set the initial values $\rho_1 = \rho_2 = 10^{-7}$ and progressively reduce the step size during iterations, since it only influences the convergence speed without impacting the final performance.

\subsection{Simulation results}
\label{sect:simu}
We generate a low rank matrix $L_0 \in \mathbb{R}^{m_1 \times m_2}$ as the product of two matrices of smaller dimensions, i.e., $L_0 = UV^\top$, where $U \in \mathbb{R}^{m_1 \times r}$, $V \in \mathbb{R}^{m_2 \times r}$ with each entry of $U$ and $V$ independently sampled from a zero-mean Gaussian distribution with the variance $1/r$, i.e., $\mathcal{N}(0,1/r)$, and consequently, the rank of $L_0$ is at most $r\ll \min(m_1, m_2)$.
The sparse matrix $S_0$ is generated by choosing a support set $\tilde{\mathcal I}$ of cardinality $k$, and setting the non-zero entries as independent samples from the uniform distribution on $[-1,1]$.
We examine two sampling schemes: uniform and non-uniform, as defined in Subsection \ref{appendix:param}, both implemented without replacement.

We define sampling ratios as $\mathrm{SR} = N / (m_1 m_2)$ and set the value of $\sigma$ in \eqref{trace regression model} such that the signal-to-noise ratio, defined by $\mathrm{SNR} = 10 \log_{10}\left(1/( N\sigma^2 ) \sum_{i=1}^{N} \langle T_i, A_0 \rangle^2 \right)$, is 20 for noisy observations. 
We explore various combinations of dimensions ($300\times 300$ or $1000\times 1000$), ranks (5 or 10), sampling schemes (uniform or non-uniform), and sampling ratios (SR $=0.2$ or 0.4) under both noise-free and noisy cases. For each combination, we select the optimal parameters for each competing method, namely L1 and TL1. 

We report the recovery results for matrices of size $300\times 300$ and $1000\times 1000$ in Table \ref{tab:sim} and Table \ref{tab: results for 1000*1000}, respectively. 
Across all settings, the proposed TL1 approach consistently outperforms L1 in terms of smaller relative errors, lower estimated ranks, and more accurate identification of sparse structures. 
Furthermore, while L1 suffers a significant performance drop under non-uniform sampling relative to the uniform case, the proposed method demonstrates robustness when the sampling rate is sufficiently high (e.g., 0.4).


{\fontsize{16}{17}\selectfont 
\setlength{\tabcolsep}{3.5pt}
\begin{table}[h!]
\centering
\small
\renewcommand{\arraystretch}{1.1} 
\caption{Simulation results for matrices of dimension $300 \times 300$ under different schemes in both noisy (SNR = 20) and noise-free settings. Reported values are the mean over 100 trials, with the standard deviation shown in parentheses.}
\label{tab:sim}
\begin{tabular}{c|cccc|cccc}
\hline 
\multicolumn{9}{c}{Uniform sampling with SNR=20} \\
\hline
 & \multicolumn{4}{c}{\textbf{L1}} &  \multicolumn{4}{c}{\textbf{TL1}} \\
\hline
  (SR,$r$) & $\RE(\hat L,L_0)$ & rank($\hat L$)& $\DC(\hat S, S_0)$ & runtime (\textit{sec.}) &  $\RE(\hat L,L_0)$ &rank($\hat L$)& $\DC(\hat S, S_0)$ &  runtime (\textit{sec.}) \\
\hline

             (0.2, 5) 
            & 0.116 (0.003) & 13 (1.3) & 0.41 (0.02) &1.52
             &0.067 (0.004)& 7 (1.0) &0.85 (0.02)  &1.50 \\ 

             (0.2, 10) & 
             0.220 (0.008)& 73 (1.6) & 0.57 (0.03) &0.75
              &0.154 (0.008)&14 (1.1) &0.80 (0.02) & 1.44 \\
             
             (0.4, 5) 
             & 0.039 (0.001)& 9 (1.9) &0.87 (0.01)  &1.44
              &0.024 (0.001)& 5 (0.5) &0.91 (0.01) & 1.21 \\
                         
             (0.4, 10) 
             & 0.085 (0.002)&38 (1.5) &0.85 (0.02) &0.94
              &0.065 (0.002)&10 (0.5) &0.91 (0.01) & 1.26\\ 
\hline

\multicolumn{9}{c}{Non-uniform sampling with SNR=20} \\

\hline
& \multicolumn{4}{c}{\textbf{L1}} &  \multicolumn{4}{c}{\textbf{TL1}} \\
\hline
(SR,$r$) & $\RE(\hat L,L_0)$ & rank($\hat L$)& $\DC(\hat S, S_0)$ & runtime (\textit{sec.}) &  $\RE(\hat L,L_0)$ &rank($\hat L$)& $\DC(\hat S, S_0)$  &  runtime (\textit{sec.}) \\
\hline

            (0.2, 5) 
            & 0.612 (0.019)&27 (2.1) &0.48 (0.02) &1.19
            &0.594 (0.020)&10 (0.0) &0.87 (0.02) & 1.02  \\ 

            (0.2, 10) & 
             0.641 (0.013)&79 (1.5) &0.41 (0.04) &1.01
             &0.598 (0.014)&19 (0.3) &0.77 (0.02) & 1.16 \\
            
             (0.4, 5) 
             & 0.157 (0.011)&23 (2.0) &0.64 (0.03) &1.49
             &0.027 (0.001)&5 (0.0) &0.92 (0.01) &1.25  \\
                         
             (0.4, 10) 
             & 0.291 (0.009)&75 (1.7) &0.61 (0.02) &2.07
             &0.063 (0.007)&10 (0.0) &0.91 (0.01) & 2.15  \\ 
\hline

\multicolumn{9}{c}{Uniform sampling without noise} \\

\hline
& \multicolumn{4}{c}{\textbf{L1}} &  \multicolumn{4}{c}{\textbf{TL1}} \\
\hline
(SR,$r$) & $\RE(\hat L,L_0)$ & rank($\hat L$)& $\DC(\hat S, S_0)$ & runtime (\textit{sec.}) &  $\RE(\hat L,L_0)$ &rank($\hat L$)& $\DC(\hat S, S_0)$  &  runtime (\textit{sec.}) \\
\hline

           (0.2, 5) 
            & 0.102 (0.004)&7 (0.8) &0.51 (0.02) &1.44
            &0.050 (0.004)& 5 (0.6) & 0.90 (0.01)&   1.54 \\ 
                       
            (0.2, 10) & 
            0.187 (0.009) & 62 (2.2) & 0.64 (0.03)  &1.20
            &0.115 (0.004) &13 (0.6) &0.85 (0.01)&   1.82 \\ 
  
            (0.4, 5) 
            & 0.009 (0.003) &5 (0.0) & 0.88 (0.01) &1.74
            &0.005 (0.004)&5 (0.0) &0.94 (0.01)&   1.62\\ 
             
             (0.4, 10) 
             & 0.033 (0.004)& 13 (2.1) & 0.87 (0.02) &1.18
             & 0.024 (0.004)& 10 (0.0) &0.93 (0.01)&  1.34\\ 
\hline

\multicolumn{9}{c}{Non-uniform sampling without noise} \\

\hline
& \multicolumn{4}{c}{\textbf{L1}} &  \multicolumn{4}{c}{\textbf{TL1}} \\
\hline
(SR,$r$) & $\RE(\hat L,L_0)$ & rank($\hat L$)& $\DC(\hat S, S_0)$ & runtime (\textit{sec.}) &  $\RE(\hat L,L_0)$ &rank($\hat L$)& $\DC(\hat S, S_0)$  &  runtime (\textit{sec.}) \\
\hline

            (0.2, 5) 
            & 0.610 (0.019)&11 (0.9) &0.63 (0.03) &0.79
              &0.576 (0.027)&10 (0.4) &0.93 (0.02)&  1.10  \\ 
                                 
             (0.2, 10) & 
             0.635 (0.013)&70 (2.4) &0.49 (0.03) &0.86
               &0.595 (0.019)&19 (0.2) &0.87 (0.02)&  0.99  \\
               
             (0.4, 5) 
             & 0.136 (0.011)&10 (0.6) &0.68 (0.02) &1.37
              &0.005 (0.002)&5 (0.0) &0.93 (0.01)&  1.25  \\
              
             (0.4, 10) 
             & 0.227 (0.010)&34 (1.7) &0.70 (0.02) &1.24
               &0.023 (0.007)&10 (0.0) &0.93 (0.01)&  1.08  \\ 

\hline

\end{tabular}
\end{table}
}

{\fontsize{16}{17}\selectfont 
\setlength{\tabcolsep}{3.5pt}
\begin{table}[h!]
\centering
\small
\renewcommand{\arraystretch}{1.1} 
\caption{Simulation results for matrices of dimension $1000 \times 1000$ under different schemes in both noisy (SNR = 20) and noise-free settings. Due to time constraints, reported values are the mean over 10 trials, with the standard deviation shown in parentheses.}
\label{tab: results for 1000*1000}
\begin{tabular}{c|cccc|cccc}
\hline
\multicolumn{9}{c}{Uniform sampling wit SNR=20} \\
\hline
& \multicolumn{4}{c}{\textbf{L1}} &  \multicolumn{4}{c}{\textbf{TL1}} \\
\hline
(SR,$r$) & $\RE(\hat L,L_0)$ & rank($\hat L$)& $\DC(\hat S, S_0)$ & runtime (\textit{sec.}) &  $\RE(\hat L,L_0)$ &rank($\hat L$)& $\DC(\hat S, S_0)$  &  runtime (\textit{sec.}) \\
\hline
           (0.2,5) 
            & 0.046 (0.001)&7 (1.2) &0.82 (0.01) &13.26
            &0.021 (0.004)& 5 (0.0) & 0.92 (0.00)&   11.90 \\ 
            
             (0.2,10) & 
             0.066 (0.002) & 92 (2.6) & 0.82 (0.01) & 11.46
            &0.041 (0.004) &10 (0.0) &0.92 (0.00)&    13.26 \\
               
            (0.4,5) 
             & 0.044 (0.001) &5 (0.0) & 0.81 (0.01) & 10.73
            &0.011 (0.004)&5 (0.0) &0.93 (0.01)&   9.19\\
             
             (0.4,10) 
             & 0.062 (0.001)& 65 (2.5) & 0.87 (0.01) &10.53
            & 0.024 (0.004)& 10 (0.0) &0.93 (0.01)& 13.06\\ 
\hline
\multicolumn{9}{c}{Non-uniform sampling with SNR=20} \\
\hline
& \multicolumn{4}{c}{\textbf{L1}} &  \multicolumn{4}{c}{\textbf{TL1}} \\
\hline
(SR,$r$) & $\RE(\hat L,L_0)$ & rank($\hat L$)& $\DC(\hat S, S_0)$ & runtime (\textit{sec.}) &  $\RE(\hat L,L_0)$ &rank($\hat L$)& $\DC(\hat S, S_0)$  &  runtime (\textit{sec.}) \\
\hline

           (0.2,5) 
            & 0.603 (0.009)&10 (0.0) &0.81 (0.01) &9.12
              &0.601 (0.004)& 10 (0.0) & 0.93 (0.00) &  8.26 \\ 
                     
             (0.2,10) & 
             0.596 (0.006) & 20 (0.0) & 0.81 (0.01)  &15.58
               &0.562 (0.007) &17 (0.6) &0.91 (0.01) &12.42 \\
               
              (0.4,5) 
             & 0.064 (0.001) &5 (0.0) & 0.88 (0.00) &10.84
               &0.012 (0.003)&5 (0.0) &0.93 (0.00)&   11.68\\
             
             (0.4,10) 
             & 0.119 (0.001)& 10 (0.0) & 0.87 (0.00) &19.08
              & 0.025 (0.002)& 10 (0.0) &0.93 (0.00) &  21.56\\ 
\hline

\hline
\multicolumn{9}{c}{Uniform sampling without noise} \\
\hline
& \multicolumn{4}{c}{\textbf{L1}} &  \multicolumn{4}{c}{\textbf{TL1}} \\
\hline
(SR,$r$) & $\RE(\hat L,L_0)$ & rank($\hat L$)& $\DC(\hat S, S_0)$ & runtime (\textit{sec.}) &  $\RE(\hat L,L_0)$ &rank($\hat L$)& $\DC(\hat S, S_0)$  &  runtime (\textit{sec.}) \\
\hline

           (0.2,5) 
            & 0.017 (0.001)&5 (0.0) &0.85 (0.01) & 13.41
              &0.004 (0.001)& 5 (0.0) & 0.94 (0.00)&   11.64  \\

             (0.2,10)  
             & 0.032 (0.002) &17 (1.4) & 0.84 (0.01) &16.56
               &0.009 (0.001) &10 (0.0) &0.93 (0.01)&  11.73  \\
             
              (0.4,5) 
             &0.002 (0.000) & 5 (0.0) & 0.88 (0.01) &10.87
               &0.001 (0.000)&5 (0.0) &0.93 (0.00) &  9.37\\
             
             (0.4,10) 
             & 0.006 (0.000)& 10 (0.0) & 0.87 (0.01) &10.94
              & 0.002 (0.000)& 10 (0.0) &0.93 (0.01)& 9.06 \\
\hline
\multicolumn{9}{c}{Non-uniform sampling without noise} \\
\hline
& \multicolumn{4}{c}{\textbf{L1}} &  \multicolumn{4}{c}{\textbf{TL1}} \\
\hline
(SR,$r$) & $\RE(\hat L,L_0)$ & rank($\hat L$)& $\DC(\hat S, S_0)$ & runtime (\textit{sec.}) &  $\RE(\hat L,L_0)$ &rank($\hat L$)& $\DC(\hat S, S_0)$  &  runtime (\textit{sec.}) \\
\hline
           (0.2,5) 
            & 0.603 (0.008)&10 (0.0) &0.82 (0.01) &9.02
              &0.599 (0.001)& 10 (0.6) & 0.94 (0.00)&   10.30 \\ 
            
             (0.2,10) & 
             0.595 (0.006) & 20 (0.0) & 0.81 (0.02)  &8.92
              &0.568 (0.011) &18 (2.1) &0.94 (0.01)  & 10.31
             \\
              (0.4,5) 
             & 0.056 (0.001) &5 (0.0) & 0.89 (0.01) &10.74
               &0.001 (0.000)&5 (0.0) &0.93 (0.00)&   11.41
             \\
             
             (0.4,10) 
             & 0.108 (0.001)& 10 (0.0) & 0.88 (0.01) &10.83
               & 0.002 (0.000)& 10 (0.0) &0.93 (0.00)&  11.45
             \\ 
\hline
\end{tabular}
\end{table}
}

\textbf{Discussion 1.} When the parameters $a_1$ and 
$a_2$ are set too large, the TL1 penalty closely approximates the L1 norm, thereby diminishing the advantages of using TL1. This observation suggests us that the search range for $a_1$ and $a_2$ can be narrowed accordingly, reducing computational cost without sacrificing performance.
Specifically, for example, Table \ref{tab:tl1-params} reports the optimal parameter values under the uniform sampling setting with SNR = 20, showing that both $a_1$ and $a_2$ remain moderate in magnitude.

{\fontsize{19}{20}\selectfont 
\setlength{\tabcolsep}{9pt}
\begin{table}[h!]
\centering
\small
\caption{Optimal parameter settings under uniform sampling with SNR = 20.}
\label{tab:tl1-params}
\begin{tabular}{c c c c c c}
\toprule
$m_1=m_2$ & (rank, SR) & $\lambda_1$ & $\lambda_2$ & $a_1$ & $a_2$ \\
\midrule
300  & (5, 0.2)  & $10^{-4}$ & $10^{-6}$ & 10 & 0.1 \\
300  & (5, 0.4)  & $10^{-4}$ & $10^{-6}$ & 10 & 0.1 \\
300  & (10, 0.2) & $10^{-4}$ & $10^{-6}$ & 10 & 0.1 \\
300  & (10, 0.4) & $10^{-4}$ & $10^{-6}$ & 10 & 0.1 \\
\hline
1000 & (5, 0.2)  & $10^{-4}$ & $10^{-7}$ & 1  & 1   \\
1000 & (5, 0.4)  & $10^{-4}$ & $10^{-7}$ & 10 & 1   \\
1000 & (10, 0.2) & $10^{-4}$ & $10^{-7}$ & 1  & 1   \\
1000 & (10, 0.4) & $10^{-4}$ & $10^{-7}$ & 10 & 1   \\
\bottomrule
\end{tabular}
\end{table}
}

\textbf{Discussion 2} (Sensitivity analysis)\textbf{.} To examine the sensitivity of the proposed TL1 model to its parameters, we conduct a series of controlled experiments. In the first setting, we fix $a_1, a_2$ at their empirically optimal values and vary $\lambda_1, \lambda_2$ to assess the impact on performance in terms of relative error, Dice’s coefficient, and rank. Conversly, in the second setting, we then fix $\lambda_1$ and $\lambda_2$ at their optimal values, while varying $a_1$ and $a_2$ to evaluate their influence.
Table \ref{tab:lambda-grid} and Table \ref{tab:a-grid} report the results for both cases, respectively, under the configuration $m_1 = m_2 = 300, \, r = 5, \, \text{SR}=0.2, \, \text{SNR}=20, \, \text{uniform sampling}$. Each entry in the tables represents a tuple of values: relative error, Dice’s coefficient, and the rank 
of the recovered low-rank matrix.

Comparing Tables \ref{tab:lambda-grid} and \ref{tab:a-grid}, it is evident that the model is more sensitive to the values of $\lambda_1$ and $\lambda_2$, compared to $a_1$ and $a_2$. This is reasonable, as the $\lambda$ parameters (whether $\lambda_1$ or $\lambda_2$) correspond to the noise level: the smaller the amount of noise, the smaller their values. When $\lambda$ deviates from the optimal range, the recovered low-rank or sparse matrix may degenerate; in some cases, either the low-rank component or the sparse component becomes entirely zero, indicating a failure of decomposition. In contrast, varying $a_1$ and $a_2$ around their optimal values produces relatively minor changes, and the recovery remains stable in terms of both accuracy and sparse structure. This behavior suggests that the TL1 penalty is insensitive to $a_1$ and $a_2$, provided they lie within reasonable ranges.

{\fontsize{19}{20}\selectfont 
\setlength{\tabcolsep}{8pt}
\begin{table}[h!]
\centering
\small
\caption{Results with tuples (RE, Dice, rank) across a grid of $(\lambda_1,\lambda_2)$.}
\label{tab:lambda-grid}
\begin{tabular}{lccccc}
\toprule
$\lambda_1 \,\backslash\, \lambda_2$ & $10^{-4}$ & $10^{-5}$ & $10^{-6}$ & $10^{-7}$ & $10^{-8}$ \\
\midrule
$10^{-2}$ & (0.530, 0.00, 4) & (0.346, 0.54, 4) & (0.156, 0.35, 5) & (0.783, 0.09, 2) & (0.796, 0.08, 2) \\
$10^{-3}$ & (0.080, 0.00, 6) & (0.073, 0.53, 5) & (0.073, 0.78, 5) & (0.088, 0.52, 5) & (0.091, 0.49, 5) \\
$10^{-4}$ & (0.076, 0.00, 9) & (0.072, 0.14, 9) & \textbf{(0.056, 0.87, 6)} & (0.078, 0.64, 6) & (0.082, 0.60, 5) \\
$10^{-5}$ & (0.109, 0.00, 52) & (0.109, 0.00, 52) & (0.097, 0.70, 42) & (0.095, 0.80, 37) & (0.094, 0.82, 32) \\
$10^{-6}$ & (0.110, 0.00, 57) & (0.110, 0.00, 57) & (0.106, 0.44, 52) & (0.105, 0.69, 51) & (0.104, 0.71, 48) \\
\bottomrule
\end{tabular}
\end{table}
}

{\fontsize{19}{20}\selectfont 
\setlength{\tabcolsep}{8pt}
\begin{table}[h!]
\centering
\small
\caption{Results with tuples (RE, Dice, rank) across a grid of $(a_1,a_2)$.}
\label{tab:a-grid}
\begin{tabular}{lccccc}
\toprule
$a_1 \,\backslash\, a_2$ & 0.01 & 0.05 & 0.1 & 1 & 10 \\
\midrule
0.1   & (0.102, 0.76, 41) & (0.102, 0.77, 41) & (0.102, 0.76, 41) & (0.101, 0.80, 36) & (0.099, 0.81, 30) \\
1     & (0.062, 0.78, 9)  & (0.059, 0.83, 9)  & (0.059, 0.86, 9)  & (0.068, 0.86, 7)  & (0.071, 0.84, 7) \\
10    & (0.057, 0.84, 6)  & (0.056, 0.86, 6)  & \textbf{(0.056, 0.87, 6)} & (0.065, 0.77, 6)  & (0.070, 0.71, 5) \\
100   & (0.064, 0.84, 5)  & (0.064, 0.86, 5)  & (0.067, 0.84, 5)  & (0.092, 0.67, 5)  & (0.097, 0.60, 5) \\
1000  & (0.093, 0.84, 5)  & (0.094, 0.85, 5)  & (0.101, 0.77, 5)  & (0.148, 0.51, 5)  & (0.158, 0.43, 5) \\
\bottomrule
\end{tabular}
\end{table}
}

\subsection{Video background separation}\label{sec:realdata}

We demonstrate a real-world application of RPCA using three video datasets: one synthetic and two real-world. Each video frame is treated as a matrix in $\mathbb{R}^{w \times h}$, which is then reshaped into a column vector in $\mathbb{R}^{wh}$. By stacking these vectors from $t$ frames, we construct the data matrix $X \in \mathbb{R}^{(wh)\times t}$ and aim for its decomposition into a low-rank matrix, corresponding to a static background (BG), and a sparse matrix, corresponding to moving objects, referred to as foreground (FG). This application does not involve any sampling scheme, i.e., $T$ in \eqref{eq:full_optimization} is the all-one matrix. After RPCA by either L1 or TL1, we reshape a particular column in the recovered low-rank and sparse components, $\hat{L}$ and $\hat{S}$, back into 2D frames for visualization.


We generate a synthetic video sequence composed of a fixed random noise as background and a moving foreground object, that is a white square traveling diagonally from the top-left to the bottom-right corner over time. Three particular frames are presented in Figure \ref{fig:syn-video-appendix}. Each frame is represented as a 2D image, which we reshape into a vector. By stacking these vectors column-wise across time, we form two matrices: a rank-one matrix representing the static background and a sparse matrix encoding the moving foreground object. The sum of these two matrices yields the final data matrix, which serves as the input for RPCA. 

The visual comparison results in Figure  \ref{fig:syn-video-appendix} demonstrate that the proposed TL1 model preserve both the low-rank structure of the static background and the sparse, dynamic background, more effectively than L1. For example, the moving squares recovered by TL1 across time exhibit fewer artifacts and sharper boundaries, indicating more accurate separation. In contrast, the background estimated by the standard L1 model retains noticeable motion trails, suggesting incomplete separation and contamination by foreground elements. Moreover, we evaluate the quantitative performance in Table~\ref{tab:synthetic video result}, which shows that TL1 regularization achieves superior recovery quality compared to L1. Specifically, TL1 yields a lower relative error, recovers a background matrix with exactly rank one, and achieves higher DC values, indicating better separation of low-rank and sparse components.


\begin{figure}[htbp]
    \centering
    \begin{tabular}{cccc}
    Video frame & Ground-truth & Recovered by L1 & Recovered by TL1\\
           \includegraphics[width=0.225\textwidth]{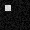} 
       & \includegraphics[width=0.225\textwidth]{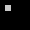}&  
         \includegraphics[width=0.225\textwidth]{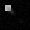} &
            \includegraphics[width=0.225\textwidth]{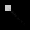}\\
            &  
       \includegraphics[width=0.225\textwidth]{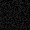} &
         \includegraphics[width=0.225\textwidth]{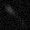}&
        \includegraphics[width=0.225\textwidth]{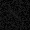}\\

        \includegraphics[width=0.225\textwidth]{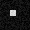} 
       & \includegraphics[width=0.225\textwidth]{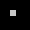}&  
         \includegraphics[width=0.225\textwidth]{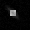} &
            \includegraphics[width=0.225\textwidth]{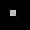}\\
            &  
       \includegraphics[width=0.225\textwidth]{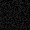} &
         \includegraphics[width=0.225\textwidth]{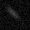}&
        \includegraphics[width=0.225\textwidth]{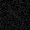}\\  
        \includegraphics[width=0.225\textwidth]{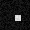} 
       & \includegraphics[width=0.225\textwidth]{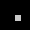}&  
         \includegraphics[width=0.225\textwidth]{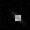} &
            \includegraphics[width=0.225\textwidth]{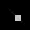}\\
            &  
       \includegraphics[width=0.225\textwidth]{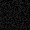} &
         \includegraphics[width=0.225\textwidth]{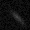}&
        \includegraphics[width=0.225\textwidth]{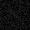}\\  

    \end{tabular}
    \caption{synthetic-video background separation: three particular frames are presented with the foreground (sparse component) shown in the odd rows and the background (low-rank component) in the even rows.}
    \label{fig:syn-video-appendix}
\end{figure}

\begin{figure}[htbp]
    \centering
  \begin{tabular}{ccccc}
  \small{Frame} & \small{FG by L1} & \small{FG by TL1} & \small{BG by L1} &\small{BG by TL1}\\
      \includegraphics[width=0.18\textwidth]{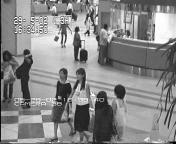} &  
              \includegraphics[width=0.18\textwidth]{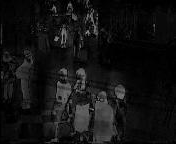}&
          \includegraphics[width=0.18\textwidth]{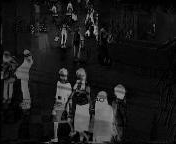}&
      \includegraphics[width=0.18\textwidth]{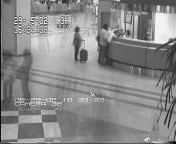}&
      \includegraphics[width=0.18\textwidth]{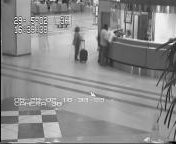}
      \\

      \includegraphics[width=0.18\textwidth]{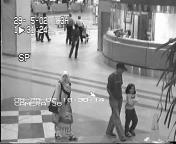} &  
              \includegraphics[width=0.18\textwidth]{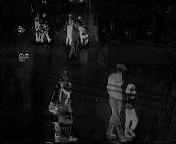}&
          \includegraphics[width=0.18\textwidth]{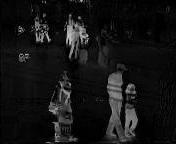}&
      \includegraphics[width=0.18\textwidth]{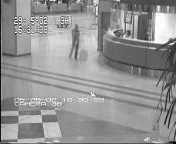}&
      \includegraphics[width=0.18\textwidth]{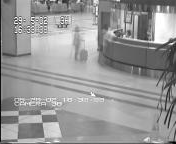}
      \\

      \includegraphics[width=0.18\textwidth]{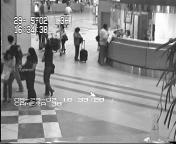} &  
              \includegraphics[width=0.18\textwidth]{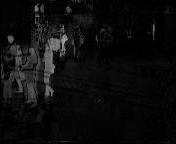}&
          \includegraphics[width=0.18\textwidth]{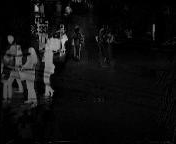}&
      \includegraphics[width=0.18\textwidth]{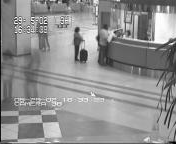}&
      \includegraphics[width=0.18\textwidth]{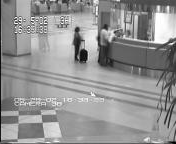}
      \\
  \end{tabular}
    \caption{Airport video background separation: three particular frames are shown.}
    \label{fig:real-video-appendix-1}
\end{figure}

\begin{figure}[htbp]
    \centering
  \begin{tabular}{ccccc}
  \small{Frame} & \small{FG by L1} & \small{FG by TL1} & \small{BG by L1} &\small{BG by TL1}\\

      \includegraphics[width=0.18\textwidth]{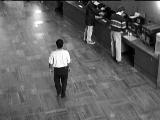} &  
              \includegraphics[width=0.18\textwidth]{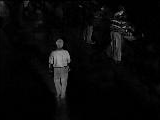}&
          \includegraphics[width=0.18\textwidth]{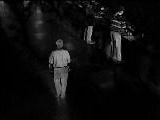}&
      \includegraphics[width=0.18\textwidth]{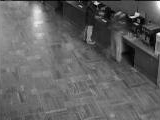}&
      \includegraphics[width=0.18\textwidth]{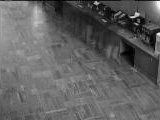}
      \\

      \includegraphics[width=0.18\textwidth]{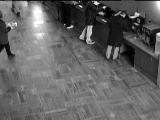} &  
              \includegraphics[width=0.18\textwidth]{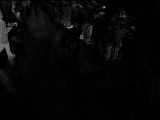}&
          \includegraphics[width=0.18\textwidth]{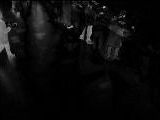}&
      \includegraphics[width=0.18\textwidth]{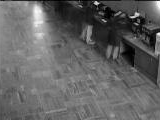}&
      \includegraphics[width=0.18\textwidth]{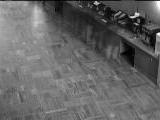}
      \\

      \includegraphics[width=0.18\textwidth]{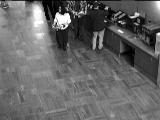} &  
              \includegraphics[width=0.18\textwidth]{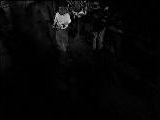}&
          \includegraphics[width=0.18\textwidth]{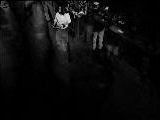}&
      \includegraphics[width=0.18\textwidth]{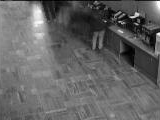}&
      \includegraphics[width=0.18\textwidth]{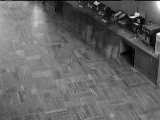}
      \\
  \end{tabular}
    \caption{Buffet restaurant video background separation: three particular frames are shown.}
    \label{fig:real-video-appendix-2}
\end{figure}

We then examine the video separation using two real video sequences, referred to as Airport and Buffet restaurant, both of which are publicly available\footnote{\url{https://sites.google.com/site/backgroundsubtraction/test-sequences/human-activities}}. Figure \ref{fig:real-video-appendix-1} shows three representative frames of the raw Airport data along with the corresponding foreground and background reconstructions using L1 or TL1. The foreground results obtained with TL1 effectively reduce false detections caused by ground shadows, clearly distinguishing true moving objects from shadow interference. In contrast, the L1-based results fail to fully suppress penumbra regions, leaving residual shadows in the background reference images. The video separation of the Buffet restaurant is presented in Figure \ref{fig:real-video-appendix-2}. The background images reconstructed by TL1 preserve the background textures (e.g., static table areas), whereas the L1 background appears blurred due to sudden illumination changes and incomplete separation of foreground objects (e.g., lingering shadows near tables). Overall, these visual comparison demonstrate that the proposed TL1 method outperforms the conventional L1 approach in shadow suppression, illumination adaptation, and complex motion scenarios.

Lastly, Table \ref{tab:time comparison} compares the computation time of the standard L1 method and the proposed TL1 method for three video sequences. Both methods exhibit similar scaling behavior with respect to matrix size, reflecting their comparable algorithmic complexity, as discussed in Section \ref{algorithm}. TL1 introduces a moderate computational overhead relative to L1, primarily due to the more intricate form of its proximal operator \eqref{eq:prox-TL1} than the soft shrinkage of L1. For the Buffet restaurant video, TL1 requires more iterations to reach convergence, resulting in a longer runtime compared to L1. This suggests that while TL1 maintains computational feasibility, its efficiency may vary depending on problem-specific convergence characteristics.

{\fontsize{19}{20}\selectfont 
\setlength{\tabcolsep}{9pt}
\begin{table}[h]
\centering
\small
\caption{Comparison of two methods on a synthetic video. 
}
\label{tab:synthetic video result}
\begin{tabular}{c|c|c|c|c}
\hline
Method & $\RE(\hat L,L_0)$  & rank($
\hat{L}$) & $\DC(\hat S,S_0)$ &Time (\textit{sec.}) \\ 
\hline
\textbf{L1} & 0.5372 &  4 & 0.6208 &19.04 \\ 
\hline
\textbf{TL1} & 0.1345 & 1 & 0.9387 &44.95 \\ 
\hline
\end{tabular}
\end{table}
}

{\fontsize{19}{20}\selectfont 
\setlength{\tabcolsep}{9pt}
\begin{table}[htbp]
    \centering
    \small
    \caption{Comparison of processing time for two methods on three video sequences.}
    \begin{tabular}{c|c|c|c|c}
    \hline
    Video Name & Resolution & Number of Frames & \textbf{L1} Time (\textit{sec.}) & \textbf{TL1} Time (\textit{sec.}) \\
    \hline
    Synthetic video &$30 \times 30$  & 300 &2.19 & 2.37 \\
    Airport & $144 \times 176$ & 400 &18.44 & 19.56 \\
    Buffet restaurant & $120 \times 160$& 400 & 14.05 & 21.78 \\
    \hline
    \end{tabular}
    \label{tab:time comparison}
\end{table}
}

\section{Conclusion and future work}
In this paper, we propose a novel TL1-regularized RPCA model, achieving effective control over low-rank and sparse matrix recovery through adjustable parameters. Thanks to TL1's interpolation between $\ell_0$/$\ell_1$,  our framework is able to estimate the rank and sparsity flexibly. 
The main contribution of our work lies in a fine-grained analysis in terms of recovery performance, showing that error upper bounds achieve a minimax optimal convergence rate up to a logarithmic factor for low-rank or approximately low-rank matrices in the absence of corruption, 
which is aligned with the ones for $\ell_1$-regularized models. 
Notably, our approach does not require strong incoherence conditions on the low-rank structure or restrictive distributional assumptions on corruptions, thereby broadening its applicability to real-world scenarios. 
In addition, we establish constant-order bounds for both rank and sparsity estimations when the underlying matrices are exactly low-rank and sparse.


Despite these advances, our current theoretical analysis does not yet establish oracle properties related to sparsity and rank. We anticipate that the sharper error bounds may be attainable, particularly in the regime where the internal parameters ($a_1$ and $a_2$) are small. These directions will be pursued in future work.


\section*{Acknowledgements}
The work of Jiayi Wang is partly supported by the National Science Foundation (DMS-2401272) and Texas Artificial Intelligence Research Institute (TAIRI). Yifei Lou is partially supported by NSF CAREER  DMS-2414705.


\newpage
\appendix

\section{Theoretical proofs}\label{append:theories}

\subsection{Auxiliary lemmas}\label{S0}

Recall the definition of $\phi_{a}$ function on any matrix $A \in \mathbb{R}^{m_1 \times m_2}$, 
\begin{align*}
    \phi_{a}(A) =   \sum_{i,j} \frac{(a+1) |A_{ij}|}{a+|A_{ij}|},
\end{align*}
where $i = 1,\dots, m_1$ and $j =1,\dots,m_2$. It is straightforward that $\phi_a$ is an increasing function with respect to each entry $|A_{ij}|$. 

\begin{lemma}[Triangle inequalities]\label{S1:lemma3tri_ineq}
For any matrix $A,B\in\mathbb{R}^{m_1 \times m_2}$, we have
\begin{align}
   & \phi_{a}(A) + \phi_{a}(B) \geq \phi_{a}(A+B), \\
    &\phi_{a}(A) - \phi_{a}(B) \leq \phi_{a}(A-B) = \phi_{a}(B-A).
\end{align}
The equalities hold when the supports of the matrices $A$ and $B$ are disjoint. 
\end{lemma}

\begin{proof}
Some simple calculations lead to the following 
    \begin{align*}
\phi_{a}(A) + \phi_{a}(B) & = \sum_{i,j} \left(\frac{(a+1)|A_{ij}|}{a+|A_{ij}|} + \frac{(a+1)|B_{ij}|}{a+|B_{ij}|}\right)\\
&= (a+1)\sum_{i,j}\frac{a(|A_{ij}|+|B_{ij}|)+2|A_{ij}||B_{ij}| }{a^2+a|A_{ij}|+a|B_{ij}|+|A_{ij}||B_{ij}|}\\
&= (a+1)\sum_{i,j}\frac{|A_{ij}|+|B_{ij}| + \frac{2|A_{ij}||B_{ij}|}{a}}{a+(|A_{ij}|+|B_{ij}|+\frac{|A_{ij}||B_{ij}|}{a})}\\
&\geq(a+1) \sum_{i,j}\frac{(|A_{ij}|+|B_{ij}| + \frac{|A_{ij}||B_{ij}|}{a})}{a+(|A_{ij}|+|B_{ij}|+\frac{|A_{ij}||B_{ij}|}{a})} \\
&\geq (a+1)\sum_{i,j}\frac{(|A_{ij}|+|B_{ij}|)}{a+(|A_{ij}|+|B_{ij}|)} \\
&\geq (a+1)\sum_{i,j}\frac{|A_{ij}+B_{ij}|}{a+(|A_{ij}+B_{ij}|)}
= \phi_{a}(A+B),
\end{align*}
where the last inequality follows from the triangle inequality and the monotonicity  of $\phi_a(\cdot)$ on $[0,\infty)$. The equality holds when $|A_{ij}||B_{ij}|=0$ for all $(i,j),$ which is equivalent to $A$ and $B$ having disjoint supports.
Similarly, we can obtain
\begin{align*}
    \phi_{a}(A) - \phi_{a}(B) & = \sum_{i,j} \left(\frac{(a+1)|A_{ij}|}{a+|A_{ij}|} - \frac{(a+1)|B_{ij}|}{a+|B_{ij}|}\right)\\
&= a(a+1)\sum_{i,j}\frac{|A_{ij}|-|B_{ij}|}{a^2+a|A_{ij}|+a|B_{ij}|+|A_{ij}||B_{ij}|}\\
&\leq a(a+1)\sum_{i,j}\frac{|A_{ij}-B_{ij}|}{a^2+a|A_{ij}|+a|B_{ij}|+|A_{ij}||B_{ij}|}\\
&= (a+1)\sum_{i,j}\frac{|A_{ij}-B_{ij}|}{a + |A_{ij}| +|B_{ij}| + \frac{|A_{ij}||B_{ij}|}{a}}\\
&\leq (a+1)\sum_{i,j}\frac{|A_{ij}-B_{ij}|}{a + |A_{ij}-B_{ij}|} = \phi_{a}(A-B)\\
&=\sum_{i,j}\frac{(a+1)|B_{ij} - A_{ij}|}{a + |B_{ij} - A_{ij}|} = \phi_{a}(B-A).
\end{align*}

\end{proof}

For any matrix $A$, let  $P_{A}^\perp$ be the projector onto the complement of the support of $A$. 

\begin{lemma}\label{lemma_addition}
    For any two  matrices $A$ and $B$,  one has
    \begin{align}
        \phi_{a}(A + P_{A}^\perp(B)) = \phi_{a}(A) + \phi_{a}(P_{A}^\perp(B)).
    \end{align}
\end{lemma}

\begin{proof}
Since $A$ and $P^\perp_{A}(B)$ are matrices with disjoint supports (i.e., their non-zero entries do not overlap), the result follows directly from Lemma \ref{S1:lemma3tri_ineq}.
This implies that the summation in $\phi_{a}$ effectively separates into two independent components.


\end{proof}

Define a constraint set:
\begin{align*} 
    &\mathcal{K}(\zeta, \gamma):= \bigg\{ A\in \mathbb{R}^{m_1 \times m_2}: ||A||_{\infty} \leq \zeta, \frac{||A||_*}{\sqrt{m_1m_2}} \leq \gamma \bigg\},
\end{align*}
where $\zeta$ and $\tau$ are positive constants, and let 
\begin{align*}
    Z_{\zeta,\gamma} := \underset{A \in \mathcal{K}(\zeta, \gamma)}{\text{sup}} \bigg|\frac{1}{n} \sum_{i \in \Omega} \langle T_i, A \rangle^2 - \|A_{\mathcal{I}}\|^2_{L_2(\Pi)} \bigg|.
\end{align*}

Define $\Sigma = \frac{1}{N} \sum_{i \in \Omega} \xi_i T_i, \ W = \frac{1}{N} \sum_{i \in \Omega} T_i,$ and $\Sigma_R = \frac{1}{n} \sum_{i \in \Omega} \epsilon_i T_i$ for i.i.d.~Rademacher variables $\epsilon_i$. Note that $\Sigma $ and $W$ are normalized sums over N, while $\Sigma_R$ is over $n.$ 
\medskip

\begin{lemma}\label{lemma_ineq:A}
   Under the Assumptions \ref{assump1} and \ref{assump2},   for any $A \in \mathcal{K}(\zeta, \gamma)$,  the following inequality 
    \begin{align}
    \frac{1}{n}  \sum_{i\in \Omega} \langle T_i, A \rangle^2 \geq \|A_{\mathcal{I}}\|^2_{L_2(\Pi)} - \zeta^2 \sqrt{\frac{\log d}{n}} - \zeta \frac{\|A\|_*}{\sqrt{m_1m_2}} \sqrt{\frac{GM\log d}{n}}, \label{upper_bound_L2pi}
    \end{align}
    holds with probability at least $1-\kappa/d$, where $G$ is a constant defined in Assumption \ref{assump1} and $\kappa$ depends on a universal constant $K$.
\end{lemma}

The proof of the upper bound in \eqref{upper_bound_L2pi} closely follows that in \cite[Lemma 2]{zhao2025noisy} and is therefore omitted.

\medskip
\begin{lemma}\label{lemma_L2Pi}
Suppose Assumptions \ref{assump1} - \ref{assump3} hold, 
we further assume that $S_0 \in \mathbb{R}^{m_1 \times m_2}$ is exactly sparse, i.e., $\|S_0\|_0 \leq s_0$ for a small integer $s_0$,  and $\|L_0\|_\infty \leq \zeta$, $\|S_0\|_\infty \leq \zeta$ for the same constant $\zeta$ in \eqref{est:hat_LS}. Take $\lambda_2^{-1} = \mathcal{O} \left( \left(\frac{a_2+\zeta}{a_2+1} (\sigma\vee\zeta)\frac{\log d}{N}\right)^{-1}\right)$, then for $\lambda_1, a_1, a_2>0$, the estimator $\hat{L}$ from \eqref{est:hat_LS} satisfies  
    \begin{multline}
    \|(\hat{L} - L_0)_{\mathcal{I}}\|^2_{L_2(\Pi)} \lesssim \beta\frac{\sigma}{\sqrt{m_1m_2}} \sqrt{\frac{G d\log d}{N}} \|\hat{L} - L_0\|_* + \beta\Delta_{S_0}(N,m_1,m_2) + \zeta^2 \sqrt{\frac{\log d}{n}}\\  +\beta\lambda_1\Phi_{a_1} (L_0) - \beta\lambda_1 \Phi_{a_1} (\hat{L}) 
    +\zeta \frac{\|\hat{L} - L_0\|_*}{\sqrt{m_1m_2}} \sqrt{\frac{GM\log d}{n}},
\end{multline}
    with probability at least $1-(\kappa + 3) / d$, where $\kappa$ depends on a universal constant $K$.
\end{lemma}

\begin{proof}
The optimality inequality, $Q(\hat{L}, \hat{S}) \leq Q(L_0, \hat{S})$, yields 
\begin{align*}
&\frac{1}{N} \sum_{i=1}^{N} (Y_i - \langle T_i, \hat{L}+\hat{S} \rangle)^2 + \lambda_1 \Phi_{a_1} (\hat{L}
) + \lambda_2 \phi_{a_2} (\hat S)\\
&\quad \leq \frac{1}{N} \sum_{i=1}^{N} (Y_i - \langle T_i, L_0+ \hat{S} \rangle)^2 +  \lambda_1 \Phi_{a_1} (L_0) + \lambda_2 \phi_{a_2} (\hat S).
\end{align*}
By plugging in the trace regression model \eqref{trace regression model} for $Y_i$ and canceling the common term of $\lambda_2 \phi_{a_2} (\hat S),$ we obtain 
\begin{align*}
\frac{1}{N} \sum_{i=1}^{N} (\langle T_i, S_0 - \hat{S} \rangle + \langle T_i, L_0 - \hat{L} \rangle + \sigma\xi_i)^2 
 \leq \frac{1}{N} \sum_{i=1}^{N} (\langle T_i, S_0 - \hat{S} \rangle + \sigma\xi_i)^2 + \lambda_1 \Phi_{a_1} (L_0) - \lambda_1 \Phi_{a_1} (\hat{L}) , 
\end{align*}
which is equivalent to
\begin{align}
&\frac{1}{N} \sum_{i=1}^{N} \langle T_i, \hat{L} -L_0 \rangle^2 + \frac{2}{N} \sum_{i=1}^{N} \langle T_i, \hat{L} -L_0 \rangle \langle T_i, \hat{S} -S_0 \rangle \notag\\
&\quad \leq \frac{2\sigma}{N} \sum_{i=1}^{N} \langle T_i \xi_i , \hat{L} -L_0  \rangle  + \lambda_1 \Phi_{a_1} (L_0) - \lambda_1 \Phi_{a_1} (\hat{L}) .\label{basic_ineq_L}
\end{align}
By decomposing the summation into $\Omega$ and $\tilde{\Omega}$, we can derive
\begin{align}
&\frac{1}{N} \sum_{i\in \Omega} \langle T_i, \hat{L} -L_0 \rangle^2 + \frac{1}{N} \sum_{i\in \Tilde{\Omega}} \langle T_i, \hat{L} -L_0 \rangle^2 + \frac{2}{N} \sum_{i \in \Tilde{\Omega}} \langle T_i, \hat{L} -L_0 \rangle \langle T_i, \hat{S} -S_0 \rangle + \frac{2\sigma}{N}  \sum_{i\in \Tilde{\Omega}} \langle T_i \xi_i , L_0 - \hat{L}  \rangle  \notag\\
&\quad \leq \frac{2\sigma}{N} \sum_{i\in \Omega} \langle T_i \xi_i , \hat{L} -L_0  \rangle + \frac{2}{N} \sum_{i \in \Omega} |\langle T_i, \hat{L} -L_0 \rangle \langle T_i, \hat{S} -S_0 \rangle| + \lambda_1 \Phi_{a_1} (L_0) - \lambda_1 \Phi_{a_1} (\hat{L}). \notag\\
\end{align}
 By Cauchy's inequality and duality between operator norm and nuclear norm,  we obtain
\begin{align*}
&\frac{1}{N} \sum_{i\in \Omega} \langle T_i, \hat{L} -L_0 \rangle^2 + \frac{1}{N} \sum_{i\in \Tilde{\Omega}} \langle T_i, \hat{L} -L_0 \rangle^2 - \frac{1}{N} \sum_{i\in \Tilde{\Omega}} \langle T_i, \hat{S} -S_0 \rangle^2  - \frac{\sigma^2}{N} \sum_{i\in \Tilde{\Omega}} \xi_i^2 - \frac{2}{N} \sum_{i\in \Tilde{\Omega}} \langle T_i, \hat{L} -L_0 \rangle^2 \notag\\
&\quad \leq 2\|\Sigma\| \|(\hat{L} - L_0)_{\mathcal{I}}\|_* + \frac{2}{\beta} \sqrt{\frac{1}{n} \sum_{i\in \Tilde{\Omega}} \langle T_i, \hat{S} -S_0 \rangle^2 } \sqrt{\frac{1}{n} \sum_{i\in \Tilde{\Omega}} \langle T_i, \hat{L} -L_0 \rangle^2 } + \lambda_1 \Phi_{a_1} (L_0) - \lambda_1 \Phi_{a_1} (\hat{L}), \notag\\
\end{align*}
which can be rearranged as
\begin{align*}
&\frac{1}{N} \sum_{i\in \Omega} \langle T_i, \hat{L} -L_0 \rangle^2 
\leq 2\|\Sigma\| \|(\hat{L} - L_0)_{\mathcal{I}}\|_* + \frac{2}{\beta} \sqrt{\frac{1}{n} \sum_{i\in \Tilde{\Omega}} \langle T_i, \hat{S} -S_0 \rangle^2 } \sqrt{\frac{1}{n} \sum_{i\in \Tilde{\Omega}} \langle T_i, \hat{L} -L_0 \rangle^2 } \notag\\
&\qquad\qquad + \frac{\sigma^2}{N} \sum_{i\in \Tilde{\Omega}} \xi_i^2 + \frac{1}{N} \sum_{i\in \Tilde{\Omega}} \langle T_i, \hat{S}-S_0 \rangle^2 + \frac{1}{N} \sum_{i\in \Tilde{\Omega}} \langle T_i, \hat{L}-L_0 \rangle^2
+\lambda_1 \Phi_{a_1} (L_0) - \lambda_1 \Phi_{a_1} (\hat{L}). 
\end{align*}

Using $\langle T_i, \hat{S} - S_0\rangle \leq \|\hat{S} - S_0\|_\infty \leq 2\zeta, \langle T_i, \hat{L} - L_0\rangle \leq \|\hat{L} - L_0\|_\infty \leq 2\zeta$, and $\sum_{i\in \Tilde{\Omega}} \xi_i^2 \lesssim |\Tilde{\Omega}| \log d$ by \cite[Eq (27)]{klopp2017robust}, we have
\begin{align*}
    \frac{1}{N} \sum_{i\in \Omega} \langle T_i, \hat{L} -L_0 \rangle^2 &\lesssim 2\|\Sigma\| \|(\hat{L} - L_0)_{\mathcal{I}}\|_* + \frac{2}{\beta} \sqrt{\frac{1}{n} \sum_{i\in \Tilde{\Omega}} \langle T_i, \hat{S} -S_0 \rangle^2 } \sqrt{\frac{1}{n} \sum_{i\in \Tilde{\Omega}} \langle T_i, \hat{L} -L_0 \rangle^2 } \notag\\
    &\qquad\qquad + \frac{\sigma^2 |\Tilde{\Omega}| \log d}{N} + \frac{8\zeta^2 |\Tilde{\Omega}|}{N} +\lambda_1 \Phi_{a_1} (L_0) - \lambda_1 \Phi_{a_1} (\hat{L}), \notag\\
\end{align*}
which can be rewritten as
    \begin{align}
    \frac{1}{n} \sum_{i\in \Omega} \langle T_i, \hat{L} -L_0 \rangle^2 &\lesssim 2\beta\|\Sigma\| \|(\hat{L} - L_0)_{\mathcal{I}}\|_* + 2 \sqrt{\frac{1}{n} \sum_{i\in \Tilde{\Omega}} \langle T_i, \hat{S} -S_0 \rangle^2 } \sqrt{\frac{1}{n} \sum_{i\in \Tilde{\Omega}} \langle T_i, \hat{L} -L_0 \rangle^2 } \notag\\
    &\qquad\qquad + \beta\frac{\sigma^2 |\Tilde{\Omega}| \log d}{N} + \beta\frac{8\zeta^2 |\Tilde{\Omega}|}{N} +\beta\lambda_1 \Phi_{a_1} (L_0) - \beta\lambda_1 \Phi_{a_1} (\hat{L}). \label{ineq_for_L}
\end{align}


By the inequality \eqref{ineq_for_S2} we obtain later when proving for Theorem \ref{thm:gen_upp_S}, 
we have
\begin{align}
     \frac{1}{n} \sum_{i\in \Omega} \langle T_i, \hat{S} -S_0 \rangle^2 
     &\lesssim \beta \left(s_0(\sigma\vee\zeta)^2\frac{\log d}{N} + \min\left\{N\lambda_2\phi_{a_2}(S_0), N^2 \left(\frac{a_2+1}{a_2}\right)^2s_0\right\}\right) \notag \\
     &\asymp \beta \Delta_{S_0}(N,m_1,m_2). \notag
 \end{align}



Plugging the above result into   \eqref{ineq_for_L} yields
\begin{multline}
     \frac{1}{n} \sum_{i\in \Omega} \langle T_i, \hat{L} -L_0 \rangle^2 \lesssim \beta\|\Sigma\| \|(\hat{L} - L_0)_{\mathcal{I}}\|_* + \sqrt{\beta \Delta_{S_0}(N,m_1,m_2)} \sqrt{\frac{1}{n} \sum_{i\in \Tilde{\Omega}} \langle T_i, \hat{L} -L_0 \rangle^2 } \notag\\
    + s_0(\sigma \vee \zeta)^2\frac{ \log d}{n}  +\beta\lambda_1\Phi_{a_1} (L_0) - \beta\lambda_1 \Phi_{a_1} (\hat{L}) \notag\\
    \lesssim \beta\|\Sigma\| \|(\hat{L} - L_0)_{\mathcal{I}}\|_* + s_0(\sigma\vee\zeta)^2\frac{\log d}{n} 
    + \lambda_2\phi_{a_2}(S_0) + \beta\lambda_1\Phi_{a_1} (L_0) - \beta\lambda_1 \Phi_{a_1} (\hat{L}). \notag
 \end{multline}

By Lemma \ref{lemma_ineq:A} , we have
\begin{multline}\label{ineq:basic_for_L}
    \|(\hat{L} - L_0)_{\mathcal{I}}\|^2_{L_2(\Pi)} 
    \lesssim \beta\|\Sigma\| \|(\hat{L} - L_0)_{\mathcal{I}}\|_* + \beta \Delta_{S_0}(N,m_1,m_2) + \beta\lambda_1\Phi_{a_1} (L_0) - \beta\lambda_1 \Phi_{a_1} (\hat{L})\\
    + \zeta^2 \sqrt{\frac{\log d}{n}}
    +\zeta \frac{\|\hat{L} - L_0\|_*}{\sqrt{m_1m_2}} \sqrt{\frac{ G M\log d}{n}}.
\end{multline}
It further follows from \cite[Lemma 5 ]{klopp2014noisy} that $\|\Sigma\| \lesssim \frac{\sigma}{\sqrt{m_1m_2}} \sqrt{\frac{G d\log d}{N}}$ and hence we have
\begin{multline}\label{ineq:end-lemma4}
    \|(\hat{L} - L_0)_{\mathcal{I}}\|^2_{L_2(\Pi)} \lesssim \beta\frac{\sigma}{\sqrt{m_1m_2}} \sqrt{\frac{Gd\log d}{N}} \|\hat{L} - L_0\|_* + \beta \Delta_{S_0}(N,m_1,m_2) + \zeta^2 \sqrt{\frac{\log d}{n}} \\  +\beta\lambda_1\Phi_{a_1} (L_0) - \beta\lambda_1 \Phi_{a_1} (\hat{L})  
    +\zeta \frac{\|\hat{L} - L_0\|_*}{\sqrt{m_1m_2}} \sqrt{\frac{GM\log d}{n}}. 
\end{multline}
\end{proof}

\begin{lemma}\label{lemma_for_lowrankness}
    Assume the rank of $L_0 \in \mathbb{R}^{m_1\times m_2}$ is at most $r_0$ and $S_0 \in \mathbb{R}^{m_1\times m_2}$ is exactly sparse, take $\lambda_1^{-1} = \mathcal{O} \left( \left( \frac{(\sigma \vee \zeta)}{\sqrt{m_1m_2}} \frac{a_1 + \zeta\sqrt{m_1m_2}}{a_1+1} \sqrt{\frac{G d\log d}{n}} \right)^{-1} \right)$, then
    for any $$a_1 = {\scriptstyle \mathcal{O}} \left((a_1+1)\lambda_1\left((\Delta_{L_0}(n, m_1, m_2) + \Delta_{S_0}(N, m_1, m_2))^2 G d\log d/( nm_1m_2)\right)^{-1/4}\right),$$ there exists a constant $c>0$ such that the smallest non-zero singular value of the estimator $\hat{L}$ from \eqref{est:hat_LS} is greater than or equal to 
    $$c \sqrt{\lambda_1(a_1^2+a_1)}\left((\Delta_{L_0}(n, m_1, m_2) + \Delta_{S_0}(N, m_1, m_2))^2 G d\log d/( nm_1m_2)\right)^{-1/8}.$$
\end{lemma}

\begin{proof}
    Assume the true rank of $L_0$ is $r_0$ and the rank of the estimator $\hat{L}$ is $k\leq m$. We only discuss the case where $k\geq r_0;$ as $k\leq r_0$ is oracle. Let $\{u_j\}$ and $\{v_j\}$ denote the left and right orthonormal singular vectors of $\hat L$, respectively, and let $D = \text{diag}(\sigma_1,\dots,\sigma_m)$ be the diagonal matrix of its the singular values arranged in a decreasing order. Then by SVD, we have $\hat{L} = \sum_{j=1}^m \sigma_ju_jv_j^\intercal$. Similarly, we denote the SVD of $L_0 = \sum_{j=1}^m \sigma_j^* u_j^* v_j^{*^\intercal}$.

    Next, we derive the partial derivative of $Q(\hat{L}, \hat{S})$ with respect to any singular values $\sigma_s$ where $s>k$:
    \begin{align}
        &\frac{\partial Q(\hat{L}, \hat{S})}{\partial \sigma_s} = \frac{2}{N} \sum_{i=1}^N (\langle T_i, \hat{L} + \hat{S} \rangle - Y_i) \langle T_i, u_s v_s^\intercal \rangle + \lambda_1 \frac{a_1(a_1+1)}{(a_1+ \sigma_s)^2}\notag\\
        =& \frac{2}{N} \sum_{i=1}^N \bigg((\langle T_i, L_0 + S_0 \rangle - Y_i) \langle T_i, u_sv_s^\intercal \rangle +  (\langle T_i, \hat{L} - L_0 + \hat{S} - S_0 \rangle) \langle T_i, u_sv_s^\intercal \rangle\bigg) + \lambda_1 \frac{a_1(a_1+1)}{(a_1+ \sigma_s)^2}\notag\\
        =& \frac{2}{N}\sum_{i\in \Omega} \sigma\xi_i \langle T_i, u_sv_s^\intercal \rangle +  \frac{2}{N}\sum_{i\in \tilde{\Omega}} \sigma\xi_i \langle T_i, u_sv_s^\intercal \rangle + \frac{2}{N} \sum_{i=1}^N \langle T_i, \hat{L}-L_0\rangle \langle T_i, u_sv_s^\intercal \rangle \notag\\
        &\qquad\qquad\qquad\qquad+ \frac{2}{N} \sum_{i=1}^N \langle T_i, \hat{S}-S_0\rangle \langle T_i, u_sv_s^\intercal \rangle + \lambda_1 \frac{a_1(a_1+1)}{(a_1+ \sigma_s)^2}\notag\\
        \lesssim &2\langle \Sigma, u_s v_s^\intercal \rangle + \frac{\sigma^2}{N} \sum_{i \in \tilde{\Omega}}\xi_i^2 + \frac{1}{N} \sum_{i \in \tilde{\Omega}} \langle T_i, u_s v_s^\intercal \rangle^2\notag \\
        &\qquad+ 2\sqrt{\frac{1}{N}\sum_{i=1}^N \langle T_i, \hat{L}-L_0 \rangle^2 +  \frac{1}{N}\sum_{i=1}^N \langle T_i, \hat{S}-S_0 \rangle^2}\sqrt{\frac{1}{N}\sum_{i=1}^N \langle T_i, u_sv_s^\intercal \rangle^2} + \lambda_1 \frac{a_1(a_1+1)}{(a_1+ \sigma_s)^2}\notag\\
        =& (*) + \lambda_1 \frac{a_1(a_1+1)}{(a_1+ \sigma_s)^2}. \label{ineq:der-L-last}
    \end{align}
    
    By  \cite[Lemma 10]{klopp2017robust}, the duality between the operator norm and nuclear norm, and the fact that $u_sv_s^\intercal$ is a rank-1 matrix, we have 
    \begin{align*}
        \langle \Sigma, u_sv_s^\intercal \rangle \leq \|\Sigma\|\|u_sv_s^\intercal\|_* = \|\Sigma\|\|u_s\|\|v_s^\intercal\| \lesssim \sqrt{\frac{G d\log d}{Nm_1m_2}},
    \end{align*}
    holds with probability at least $1/d$.

    Additionally, using the result of \cite[Eq (27)]{klopp2017robust} together with $|\langle T_i, u_s v_s^\intercal \rangle|\leq 1,$ we get
    \begin{align*}
        \frac{\sigma^2}{N} \sum_{i \in \tilde{\Omega}}\xi_i^2 + \frac{1}{N} \sum_{i \in \tilde{\Omega}} \langle T_i, u_s v_s^\intercal \rangle^2 \lesssim \frac{\sigma^2\log d}{N}s_0+ \frac{s_0}{N} \lesssim \frac{\log d}{N},
    \end{align*}
    and it is straightforward to have
    \begin{align*}
        \frac{1}{N}\sum_{i=1}^N \langle T_i, \hat{L}-L_0 \rangle^2 +  \frac{1}{N}\sum_{i=1}^N \langle T_i, \hat{S}-S_0 \rangle^2 \lesssim \Delta_{L_0}(n, m_1, m_2) + \Delta_{S_0}(N, m_1, m_2),
    \end{align*}
   by the definitions of $\Delta_{L_0}(n, m_1, m_2) $ and $\Delta_{S_0}(N, m_1, m_2).$

    Similar to the discussion in \cite[Lemma 5]{zhao2025noisy}  regarding $\frac{1}{N}\sum_{i=1}^N \langle T_i, u_s v_s^\intercal \rangle^2$, we have
    \begin{align*}
        \frac{1}{N}\sum_{i=1}^N \langle T_i, u_s v_s^\intercal \rangle^2 = \frac{1}{N}\sum_{i\in\Omega} \langle T_i, u_s v_s^\intercal \rangle^2 + \frac{1}{N}\sum_{i\in \tilde{\Omega}} \langle T_i, u_s v_s^\intercal \rangle^2 \lesssim \sqrt{\frac{Gd\log d}{nm_1m_2}} + \frac{\log d}{N} \lesssim \sqrt{\frac{Gd\log d}{nm_1m_2}},
    \end{align*}
which leads to
    \begin{align*}
        (*)&\lesssim \sqrt{\frac{Gd\log d}{Nm_1m_2}} + \frac{\log d}{N} + \sqrt{\Delta_{L_0}(n, m_1, m_2) + \Delta_{S_0}(N, m_1, m_2)}\left(\frac{Gd\log d}{nm_1m_2}\right)^{1/4}\\ 
        &\lesssim \sqrt{\Delta_{L_0}(n, m_1, m_2) + \Delta_{S_0}(N, m_1, m_2)}\left(\frac{Gd\log d}{nm_1m_2}\right)^{1/4}.
    \end{align*}

    Comparing it with the last term in the derivative \eqref{ineq:der-L-last}, we obtain
    \begin{align}
        \frac{\lambda_1 \frac{a_1(a_1+1)}{(a_1+ \sigma_s)^2}}{(*)} \gtrsim \lambda_1 \frac{a_1(a_1+1)}{(a_1+ \sigma_s)^2}(\Delta_{L_0}(n, m_1, m_2) + \Delta_{S_0}(N, m_1, m_2))^{-1/2}\left(\frac{Gd\log d}{nm_1m_2}\right)^{-1/4}.
    \end{align}
    When $\sigma_s \lesssim \sqrt{\lambda_1(a_1^2+a_1)}\left((\Delta_{L_0}(n, m_1, m_2) + \Delta_{S_0}(N, m_1, m_2))^2 G d\log d/( nm_1m_2)\right)^{-1/8}$, we have    
    for any $a_1 = {\scriptstyle \mathcal{O}} \left((a_1+1)\lambda_1\left((\Delta_{L_0}(n, m_1, m_2) + \Delta_{S_0}(N, m_1, m_2))^2 G d\log d/( nm_1m_2)\right)^{-1/4}\right)$,  
    \begin{align*}
        \frac{\lambda_1 \frac{a_1(a_1+1)}{(a_1+ \sigma_s)^2}}{(*)} &\gtrsim \lambda_1 \frac{a_1(a_1+1)}{(a_1+ \sigma_s)^2}(\Delta_{L_0}(n, m_1, m_2) + \Delta_{S_0}(N, m_1, m_2))^{-1/2}\left(\frac{Gd\log d}{nm_1m_2}\right)^{-1/4}  \\
        &\gtrsim \lambda_1 \frac{a_1(a_1+1)}{ \sigma_s^2}(\Delta_{L_0}(n, m_1, m_2) + \Delta_{S_0}(N, m_1, m_2))^{-1/2}\left(\frac{Gd\log d}{nm_1m_2}\right)^{-1/4} \gtrsim 1,
    \end{align*}
    which implies $\frac{\partial Q(\hat{L}, \hat{S})}{\partial \sigma_s} > 0$. Then it follows from \cite[Lemma 1]{fan2001variable} that there exists a constant $c>0$ such that $$\sigma_s \geq c \sqrt{\lambda_1(a_1^2+a_1)}\left((\Delta_{L_0}(n, m_1, m_2) + \Delta_{S_0}(N, m_1, m_2))^2 G d\log d/( nm_1m_2)\right)^{-1/8},$$ for any $s>k$.
\end{proof}

\begin{lemma}\label{lemma_for_sparsity}
    Assume $L_0$ is approximately or exactly low-rank and $S_0$ is a sparse matrix, when $\lambda_2^{-1} = \mathcal{O}\bigg( \bigg(\frac{a_2+\zeta}{a_2+1} (\sigma \vee \zeta) \frac{\log d}{N} \bigg)^{-1} \bigg)$, for any $a_2 = {\scriptstyle \mathcal{O}}\left(\lambda_2(\Delta_{L_0}(n, m_1, m_2) + \Delta_{S_0}(N, m_1, m_2))^{-1/2}\right)$, there exits a positive constant $c'$ such that for any $(k, l)$-th non-zero entry of the estimator $\hat{S} \in \mathbb{R}^{m_1 \times m_2}$ should satisfy 
    $$|\hat{S}_{kl}| \geq c'\sqrt{\lambda_2(a_2^2+a_2)}(\Delta_{L_0}(n, m_1, m_2) + \Delta_{S_0}(N, m_1, m_2))^{-1/4},$$ 
    for any $(k, l) \notin \tilde{\mathcal{I}}$. 
\end{lemma}

\begin{proof}

For the estimated pair $(\hat{L}, \hat{S})$, the partial derivative of $Q(\hat{L}, \hat{S})$ with respect to non-zero $|\hat{S}_{kl}|$, where $(k, l) \notin \tilde{\mathcal{I}}$, which implies that $i \notin \tilde{\Omega}$, would be 
\begin{align}
   & \frac{\partial Q(\hat{L}, \hat{S})}{\partial |\hat{S}_{kl}|} = \frac{2}{N} \sum_{i=1}^N (\langle T_i, \hat{L} + \hat{S} \rangle - Y_i) \langle T_i, \frac{\partial \hat{S}}{\partial |\hat{S}_{kl}|} \rangle + \lambda_2 \frac{a_2(a_2+1)}{(a_2+ |\hat{S}_{kl}|)^2} \notag\\
    = &\frac{2}{N} \sum_{i=1}^N (\langle T_i, L_0 + S_0 \rangle - Y_i) \langle T_i, \frac{\partial \hat{S}}{\partial |\hat{S}_{kl}|} \rangle + \frac{2}{N} \sum_{i=1}^N (\langle T_i, \hat{L} - L_0 + \hat{S} - S_0 \rangle) \langle T_i, \frac{\partial \hat{S}}{\partial |\hat{S}_{kl}|} \rangle + \lambda_2 \frac{a_2(a_2+1)}{(a_2+ |\hat{S}_{kl}|)^2}\notag\\
    =&\frac{2}{N} \sum_{i \in \Omega} (\langle T_i, L_0 + S_0 \rangle - Y_i) \langle T_i, \frac{\partial \hat{S}}{\partial |\hat{S}_{kl}|} \rangle + \frac{2}{N} \sum_{i=1}^N (\langle T_i, \hat{L} - L_0 + \hat{S} - S_0 \rangle) \langle T_i, \frac{\partial \hat{S}}{\partial |\hat{S}_{kl}|} \rangle + \lambda_2 \frac{a_2(a_2+1)}{(a_2+ |\hat{S}_{kl}|)^2}\notag\\
    = &2 \langle \Sigma, \frac{\partial \hat{S}}{\partial |\hat{S}_{kl}|} \rangle + \frac{2}{N} \sum_{i=1}^N (\langle T_i, \hat{L} - L_0 \rangle) \langle T_i, \frac{\partial \hat{S}}{\partial |\hat{S}_{kl}|} \rangle + \frac{2}{N} \sum_{i=1}^N (\langle T_i, \hat{S} - S_0 \rangle) \langle T_i, \frac{\partial \hat{S}}{\partial |\hat{S}_{kl}|} \rangle + \lambda_2 \frac{a_2(a_2+1)}{(a_2+ |\hat{S}_{kl}|)^2}\notag\\
    =&2 \langle \Sigma, \frac{\partial \hat{S}}{\partial |\hat{S}_{kl}|} \rangle +  \frac{2}{N} \sum_{i=1}^N (\langle T_i, \hat{L} - L_0 \rangle) \langle T_i, \frac{\partial \hat{S}}{\partial |\hat{S}_{kl}|} \rangle + \frac{2}{N} \sum_{i=1}^N (\langle T_i, \hat{S} - S_0 \rangle) \langle T_i, \frac{\partial \hat{S}}{\partial |\hat{S}_{kl}|} \rangle + \lambda_2 \frac{a_2(a_2+1)}{(a_2+ |\hat{S}_{kl}|)^2}\notag\\
    \geq& -2\|\Sigma\|_\infty \left\|\frac{\partial \hat{S}}{\partial |\hat{S}_{kl}|}\right\|_1 - 2\sqrt{\frac{1}{N} \sum_{i=1}^N \langle T_i, \hat{L} - L_0 \rangle^2}\sqrt{\frac{1}{N}\sum_{i=1}^N \langle T_i, \frac{\partial \hat{S}}{\partial |\hat{S}_{kl}|} \rangle^2} \notag\\
    &\qquad- 2\sqrt{\frac{1}{N} \sum_{i=1}^N \langle T_i, \hat{S} - S_0 \rangle^2}\sqrt{\frac{1}{N}\sum_{i=1}^N \langle T_i, \frac{\partial \hat{S}}{\partial |\hat{S}_{kl}|} \rangle^2}
    + \lambda_2 \frac{a_2(a_2+1)}{(a_2+ |\hat{S}_{kl}|)^2}\notag\\
    \geq& -2\|\Sigma\|_\infty - 2\sqrt{\frac{1}{N} \sum_{i=1}^N \langle T_i, \hat{L} - L_0 \rangle^2} - 2\sqrt{\frac{1}{N} \sum_{i=1}^N \langle T_i, \hat{S} - S_0 \rangle^2} + \lambda_2 \frac{a_2(a_2+1)}{(a_2+ |\hat{S}_{kl}|)^2}. \label{ineq:der-S-last}
\end{align}

By  \cite[Lemma 10]{klopp2017robust}, we have $\|\Sigma\|_\infty \lesssim (\sigma \log d)/N$, which means $\|\Sigma\|_\infty = \mathcal{O}_p(\log d/N)$. Then,
\begin{align*}
    &2\|\Sigma\|_\infty + 2 \sqrt{\frac{1}{N} \sum_{i=1}^N \langle T_i, \hat{L} - L_0 \rangle^2} + 2\sqrt{\frac{1}{N} \sum_{i=1}^N \langle T_i, \hat{S} - S_0 \rangle^2} \\
    &\quad = \mathcal{O}_p \left(\sqrt{\Delta_{L_0}(n, m_1, m_2) + \Delta_{S_0}(N, m_1, m_2)}\right).
\end{align*}
Comparing it to the last term in \eqref{ineq:der-S-last}, we have
\begin{align*}
    &\frac{\lambda_2 \frac{a_2(a_2+1)}{(a_2+ |\hat{S}_{kl}|)^2}}{2\|\Sigma\|_\infty + 2 \sqrt{\frac{1}{N} \sum_{i=1}^N \langle T_i, \hat{L} - L_0 \rangle^2} + 2\sqrt{\frac{1}{N} \sum_{i=1}^N \langle T_i, \hat{S} - S_0 \rangle^2}} \\
   \gtrsim  & \frac{\lambda_2a_2(a_2+1)}{(a_2+ |\hat{S}_{kl}|)^2 \sqrt{\Delta_{L_0}(n, m_1, m_2) + \Delta_{S_0}(N, m_1, m_2)}}.
\end{align*}



 For any $a_2 = {\scriptstyle \mathcal{O}} \left(\lambda_2 (\Delta_{L_0}(n, m_1, m_2) + \Delta_{S_0}(N, m_1, m_2))^{-1/2}\right)$, when 
 $$|\hat{S}_{kl}| \leq c'\sqrt{\lambda_2(a_2^2+a_2)}(\Delta_{L_0}(n, m_1, m_2) + \Delta_{S_0}(N, m_1, m_2))^{-1/4},$$ 
 for $(k, l) \notin \tilde{\mathcal{I}}$ we have
\begin{align*}
    &\frac{\lambda_2 \frac{a_2(a_2+1)}{(a_2+ |\hat{S}_{kl}|)^2}}{2\|\Sigma\|_\infty + 2 \sqrt{\frac{1}{N} \sum_{i=1}^N \langle T_i, \hat{L} - L_0 \rangle^2} + 2\sqrt{\frac{1}{N} \sum_{i=1}^N \langle T_i, \hat{S} - S_0 \rangle^2}} \\
    \gtrsim&  \frac{\lambda_2a_2(a_2+1)}{|\hat{S}_{kl}|^2 \sqrt{\Delta_{L_0}(n, m_1, m_2) + \Delta_{S_0}(N, m_1, m_2)}} \gtrsim 1.
\end{align*}

Therefore,  $\frac{\partial Q(\hat{L}, \hat{S})}{\partial |\hat{S}_{kl}|} > 0$ holds and it further follows from   \cite[Lemma 1]{fan2001variable} that there exists a constant $c'>0$ such that 
$$|\hat{S}_{kl}| \geq c'\sqrt{\lambda_2(a_2^2+a_2)}(\Delta_{L_0}(n, m_1, m_2) + \Delta_{S_0}(N, m_1, m_2))^{-1/4},$$ 
for $(k, l) \notin \tilde{\mathcal{I}}$.
\end{proof}

\subsection{Proof of Theorem \ref{thm:gen_upp_S}} \label{S1}

\begin{proof}[\textbf{Proof of Theorem \ref{thm:gen_upp_S}}]


Similar to the proof of Lemma \ref{lemma_L2Pi}, we proceed with the optimality inequality: $Q(\hat{L}, \hat{S}) \leq Q(\hat{L}, S_0)$, i.e.,
\begin{align}
    &\frac{1}{N} \sum_{i=1}^{N} \left(Y_i - \langle T_i, \hat{L}+\hat{S} \rangle\right)^2 
    + \lambda_1 \Phi_{a_1} (\hat{L}) + \lambda_2 \phi_{a_2} (\hat{S}) \notag\\
    & \leq \frac{1}{N} \sum_{i=1}^{N} \left(Y_i - \langle T_i, \hat{L}+ S_0 \rangle\right)^2 
    + \lambda_1 \Phi_{a_1} (\hat{L}) + \lambda_2 \phi_{a_2} (S_0),
    \end{align}
which is equivalent to 
   \begin{align}
    &\frac{1}{N} \sum_{i=1}^{N} \left(\langle T_i, S_0 - \hat{S} \rangle + \langle T_i, L_0 - \hat{L} \rangle + \sigma\xi_i\right)^2 \notag\\
    & \leq \frac{1}{N} \sum_{i=1}^{N} \left(\langle T_i, L_0 - \hat{L} \rangle + \sigma\xi_i\right)^2 
    + \lambda_2\left(\phi_{a_2} (S_0) - \phi_{a_2} (\hat{S})\right), 
    \end{align}
    after substituting $Y_i$ with the trace regression model \eqref{trace regression model}. We decompose the summation and simplify to get

\begin{align}\label{basic_ineq_for_S}
    &\frac{1}{N} \sum_{i=1}^N \langle T_i, \hat{S} -S_0 \rangle^2 + \frac{2}{N} \sum_{i \in \Tilde{\Omega}} \langle T_i, \hat{L} -L_0 \rangle \langle T_i, \hat{S} -S_0 \rangle + \frac{2\sigma}{N}  \sum_{i\in \Tilde{\Omega}} \langle T_i \xi_i , S_0 - \hat{S}  \rangle \notag\\
    & \leq \frac{2\sigma}{N} \sum_{i\in \Omega} \langle T_i \xi_i , \hat{S} -S_0  \rangle + \frac{2}{N} \sum_{i \in \Omega} |\langle T_i, \hat{L} -L_0 \rangle \langle T_i, \hat{S} -S_0 \rangle| + \lambda_2\left(\phi_{a_2} (S_0) - \phi_{a_2} (\hat{S})\right).
\end{align}

By the duality between the infinity norm and $\ell_1$ norm, together with \cite[Lemma 18]{klopp2017robust}, we have
\begin{align}
    &\frac{1}{N} \sum_{i=1}^N \langle T_i, \hat{S} -S_0 \rangle^2 - \frac{2}{N} \sum_{i\in \Tilde{\Omega}} \langle T_i, \hat{S} -S_0 \rangle^2 - \frac{1}{N} \sum_{i\in \Tilde{\Omega}} \langle T_i, \hat{L} -L_0 \rangle^2 - \frac{\sigma^2}{N}\sum_{i\in \Tilde{\Omega}}\xi_i^2 \notag\\
    &\leq 2\|\Sigma\|_\infty \|(\hat{S} - S_0)_{\mathcal{I}}\|_1 + 4\zeta\|W\|_\infty\|(\hat{S} - S_0)_{\mathcal{I}}\|_1 + \lambda_2\left(\phi_{a_2} (S_0) - \phi_{a_2} (\hat{S})\right).
\end{align}
It further follows from $\langle T_i, \hat{S} - S_0\rangle \leq \|\hat{S} - S_0\|_\infty \leq 2\zeta$, $\langle T_i, \hat{L} - L_0\rangle \leq \|\hat{L} - L_0\|_\infty \leq 2\zeta,$ and the result of \cite[Eq (27)]{klopp2017robust} that
\begin{align}
    \frac{1}{N} \sum_{i=1}^N \langle T_i, \hat{S} -S_0 \rangle^2 &\lesssim \left(2\|\Sigma\|_\infty  + 4\zeta\|W\|_\infty\right)\|(\hat{S} - S_0)_{\mathcal{I}}\|_1  \notag\\
    &\quad +\frac{8\zeta^2 |\Tilde{\Omega}|}{N}+ \frac{C\sigma^2 |\Tilde{\Omega}| \log d}{N} + \lambda_2\left(\phi_{a_2} (S_0) - \phi_{a_2} (\hat{S})\right).
\end{align}

By \cite[Lemma 10]{klopp2017robust}, there exits a positive constant $C'$, such that 
\begin{align}
    \|\Sigma\|_\infty \leq C' \sigma \frac{\log d}{N}, \quad \|W\|_\infty \leq C'\frac{\log d}{N}, \notag
\end{align}
which implies
\begin{align*}
    2\|\Sigma\|_\infty  + 4\zeta\|W\|_\infty \lesssim (\sigma \vee \zeta)\frac{\log d}{N}.
\end{align*}
We derive two upper bounds in Part 1 and Part 2 using different proof techniques, and then take their minimum to obtain the desired inequality \eqref{general_bound_S} in Theorem \ref{thm:gen_upp_S}. 

\textbf{Part 1:} we obtain the first upper bound as follows, 
\begin{align*}
   & \frac{1}{N} \sum_{i=1}^N \langle T_i, \hat{S} -S_0 \rangle^2 \\
    \lesssim& (\sigma \vee \zeta)\frac{\log d}{N} \|(\hat{S} - S_0)_{\mathcal{I}}\|_1 
    + s_0(\sigma\vee\zeta)^2\frac{\log d}{N} 
    + \lambda_2\phi_{a_2}(S_0) - \lambda_2\phi_{a_2}(\hat{S}) \\
    \lesssim & (\sigma \vee \zeta)\frac{\log d}{N} \left(\|S_0\|_1 + \|\hat{S}\|_1\right) 
    + s_0(\sigma\vee\zeta)^2\frac{\log d}{N} 
    + \lambda_2\phi_{a_2}(S_0) - \lambda_2 \frac{a_2+1}{a_2+\zeta}\|\hat{S}\|_1\\
    \lesssim& \zeta(\sigma \vee \zeta)\frac{\log d}{N}s_0  + s_0(\sigma\vee\zeta)^2\frac{\log d}{N} + \lambda_2\phi_{a_2}(S_0) +\left((\sigma \vee \zeta)\frac{\log d}{N} - \lambda_2 \frac{a_2+1}{a_2+\zeta}\right)\|\hat{S}\|_1\\
    \lesssim& (\sigma\vee\zeta)^2\frac{\log d}{N}s_0 + \lambda_2\phi_{a_2}(S_0) + \left((\sigma \vee \zeta)\frac{\log d}{N} - \lambda_2 \frac{a_2+1}{a_2+\zeta}\right)\|\hat{S}\|_1.
\end{align*}

Take $\lambda_2 \gtrsim \frac{a_2+\zeta}{a_2+1}(\sigma\vee\zeta)\frac{\log d}{N}$, then
\begin{align}
    \frac{1}{N} \sum_{i=1}^N \langle T_i, \hat{S} -S_0 \rangle^2 \lesssim s_0(\sigma\vee\zeta)^2\frac{\log d}{N} 
    + \lambda_2\phi_{a_2}(S_0). 
\end{align}



The mean squared error,  $\frac{1}{N} \sum_i \langle T_i, \hat{S} -S_0 \rangle^2 $, only involves observed entries and the regularizer $\phi_{a_2}(\cdot)$ penalizes all non-zero entries, the optimal solution $\hat{S}$ tends to have zeros at those unobserved entries to minimize the objective function. As a result, the support of $\hat{S}$ is effectively restricted to the observed indices. 
Under this  consideration, we have $\|\hat{S} - S_0\|_F^2 = \sum_{i=1}^N \langle T_i, \hat{S} -S_0 \rangle^2$, which implies
\begin{align}
    \frac{\|\hat{S} - S_0\|_F^2}{m_1m_2} \lesssim s_0(\sigma\vee\zeta)^2 \frac{\log d}{m_1m_2} + \frac{N}{m_1m_2}\lambda_2\phi_{a_2}(S_0) \label{eqn:bound11}
\end{align}

Particularly, when $\lambda_2 \asymp \frac{a_2+\zeta}{a_2+1} (\sigma \vee \zeta) \frac{\log d}{N}$,  we use $\phi_{a_2}(S_0) \leq \frac{(a_2+1)\zeta}{a_2+\zeta}s_0$ to get
\begin{align}
    \frac{\|\hat{S} - S_0\|_F^2}{m_1m_2} \lesssim (\sigma\vee\zeta)^2 s_0 \frac{\log d}{m_1m_2}.
\end{align}


\textbf{Part 2:} we derive an upper bound from an alternative approach. 
By Lemmas \ref{S1:lemma3tri_ineq} and \ref{lemma_addition}, we have
\begin{multline*}
    \phi_{a_2}(\hat{S})= \phi_{a_2}(S_0 + \hat{S}- S_0) = \phi_{a_2}(S_0 + P_{S_0}(\hat{S}- S_0) + P^\perp_{S_0}(\hat{S}- S_0)) \\
    \ge \phi_{a_2}(S_0 + P^\perp_{S_0}(\hat{S}- S_0)) - \phi_{a_2}( P_{S_0}(\hat{S}- S_0) ) = \phi_{a_2}(S_0) +  \phi_{a_2}(P^\perp_{S_0}(\hat{S}- S_0)) - \phi_{a_2}( P_{S_0}(\hat{S}- S_0) ),
\end{multline*}
thus leading to
\begin{align*}
    \frac{1}{N} \sum_{i=1}^N \langle T_i, \hat{S} -S_0 \rangle^2 
    \lesssim& (\sigma \vee \zeta)\frac{\log d}{N} \|(\hat{S} - S_0)_{\mathcal{I}}\|_1 
    + s_0(\sigma\vee\zeta)^2\frac{\log d}{N} 
    + \lambda_2\phi_{a_2}(S_0) - \lambda_2\phi_{a_2}(\hat{S}) \\
   \lesssim &(\sigma \vee \zeta)\frac{\log d}{N} \left[\|P_{S_0}(\hat{S} - S_0)_{\mathcal{I}}\|_1  +\|P^\perp_{S_0}(\hat{S} - S_0)_{\mathcal{I}}\|_1  \right]\\
   & \qquad +  s_0(\sigma\vee\zeta)^2\frac{\log d}{N}  + \lambda_2 (\phi_{a_2}( P_{S_0}(\hat{S}- S_0)) - \phi_{a_2}(P^\perp_{S_0}(\hat{S}- S_0)) ).
\end{align*}
Take $\lambda_2 \gtrsim \frac{a_2+\zeta}{a_2+1}(\sigma\vee\zeta)\frac{\log d}{N}$, then
\begin{multline}
    \frac{1}{N} \|\hat{S} - S_0\|_F^2 \lesssim (\sigma \vee \zeta)\frac{\log d}{N} \sqrt{s_0} \|\hat{S} - S_0\|_F + s_0(\sigma\vee\zeta)^2\frac{\log d}{N} + \lambda_2 \frac{a_2+1}{a_2}\sqrt{s_0}\|\hat{S} - S_0\|_F\\
    \frac{1}{N} \|\hat{S} - S_0\|_F^2 \lesssim s_0(\sigma\vee\zeta)^2\frac{\log d}{N}  + \lambda_2^2 \left( \frac{a_2+1}{a_2}\right)^2 s_0 N \\
       \frac{\|\hat{S} - S_0\|_F^2}{m_1m_2} \lesssim s_0(\sigma\vee\zeta)^2 \frac{\log d}{m_1m_2} + \frac{N^2}{m_1m_2} \lambda_2^2 \left( \frac{a_2+1}{a_2}\right)^2 s_0.\label{eqn:bound2}
\end{multline}
Combining \eqref{eqn:bound11} and \eqref{eqn:bound2}, we have
\begin{align}
      \frac{\|\hat{S} - S_0\|_F^2}{m_1m_2} \lesssim s_0(\sigma\vee\zeta)^2 \frac{\log d}{m_1m_2} + \min\left\{ \frac{N}{m_1m_2}\lambda_2\phi_{a_2}(S_0),  \frac{N^2}{m_1m_2} \lambda_2^2 \left( \frac{a_2+1}{a_2}\right)^2 s_0 \right\}. \label{ineq_for_S2}
\end{align}

\end{proof}

\subsection{Proof of Theorem \ref{thm:gen_upp_L}}\label{S2}

Here, we focus on deriving an upper bound for $\|\hat{L} - L_0\|_{F}^2 / (m_1m_2)$ as follows,
\begin{align*}
    \|\hat{L} - L_0\|_F^2 
    &\leq \|(\hat{L} - L_0)_\mathcal{I}\|_F^2 + \|(\hat{L} - L_0)_{\Tilde{\mathcal{I}}}\|_F^2 \notag \\
    &\leq \|(\hat{L} - L_0)_\mathcal{I}\|_F^2 + \sum_{(i,j) \in \Tilde{\mathcal{I}}} (\hat{L}_{ij} - L_{0,ij})^2 \notag \\
    &\leq \|(\hat{L} - L_0)_\mathcal{I}\|_F^2 + \sum_{(i,j) \in \Tilde{\mathcal{I}}} \left( \hat{L}_{ij}^2 + L_{0,ij}^2 + 2 \hat{L}_{ij} L_{0,ij} \right) \notag \\
    &\leq \|(\hat{L} - L_0)_\mathcal{I}\|_F^2 + 4\zeta^2 |\Tilde{\mathcal{I}}| \notag \\
    &\leq \nu |\mathcal{I}| \|(\hat{L} - L_0)_\mathcal{I}\|_{L_2(\Pi)}^2 + 4\zeta^2 |\Tilde{\mathcal{I}}|. \notag \\
\end{align*}
In short, we get
\begin{equation}
    \frac{\|\hat{L} - L_0\|_F^2}{m_1 m_2} 
    \leq \frac{\nu |\mathcal{I}| \|(\hat{L} - L_0)_\mathcal{I}\|_{L_2(\Pi)}^2}{m_1 m_2} + \frac{4\zeta^2 |\Tilde{\mathcal{I}}|}{m_1 m_2}. \label{ineq_L_fro}       
\end{equation}

Using the result of Lemma \ref{lemma_L2Pi}, we obtain the following analysis. 

\begin{proof}[\textbf{Proof of Theorem \ref{thm:gen_upp_L}}]
Applying the triangle inequality for the nuclear norm in \eqref{ineq:end-lemma4} yields
\begin{align*}
    \|(\hat{L} - L_0)_{\mathcal{I}}\|^2_{L_2(\Pi)} \lesssim \beta\frac{\sigma}{\sqrt{m_1m_2}} \sqrt{\frac{G d\log d}{N}} \left(\|\hat{L}\|_* + \|L_0\|_*\right) + \beta \Delta_{S_0}(N,m_1,m_2) + \zeta^2\sqrt{\frac{\log d}{d}}   \notag\\
    +\beta\lambda_1\Phi_{a_1} (L_0) - \beta\lambda_1 \Phi_{a_1} (\hat{L})  
    +\zeta \frac{\|\hat{L}\|_* + \|L_0\|_*}{\sqrt{m_1m_2}} \sqrt{\frac{G M\log d}{n}}, \notag\\
    \lesssim \left(\beta\frac{\sigma}{\sqrt{m_1m_2}} \sqrt{\frac{G d\log d}{N}} + \frac{\zeta}{\sqrt{m_1m_2}} \sqrt{\frac{G M\log d}{n}}\right)\|L_0\|_* +\beta\lambda_1\Phi_{a_1} (L_0) \notag\\
    +\left(\beta\frac{\sigma}{\sqrt{m_1m_2}} \sqrt{\frac{G d\log d}{N}} + \frac{\zeta}{\sqrt{m_1m_2}} \sqrt{\frac{G M\log d}{n}}\right)\|\hat{L}\|_* - \beta\lambda_1 \Phi_{a_1} (\hat{L})  \notag\\
    + \beta \Delta_{S_0}(N,m_1,m_2) + \zeta^2\sqrt{\frac{\log d}{d}}.
\end{align*}

Similar to the proof of Theorem \ref{thm:gen_upp_S}, we divide the discussion into two parts.

\textbf{Part 1:}
Using the inequality of TL1 function: $\Phi_{a_1} (\hat{L}) \geq \frac{a_1+1}{a_1+\sigma_1(\hat{L})} \|\hat{L}\|_*$ 
where $\sigma_1(\hat{L})$ denotes the largest singular value of $\hat{L}$, we have
\begin{align}
    \|(\hat{L} - L_0)_{\mathcal{I}}\|^2_{L_2(\Pi)} &\lesssim \beta \left\{(\sigma \vee \zeta)\sqrt{\frac{G d\log d}{n}} \frac{\|L_0\|_*}{\sqrt{m_1m_2}} + \lambda_1 \Phi_{a_1}(L_0)\right\} \notag\\
    &\qquad + \beta \Delta_{S_0}(N,m_1,m_2) +\zeta^2\sqrt{\frac{\log d}{d}}
    \notag\\
    &\qquad+\beta\left(\frac{(\sigma \vee \zeta)}{\sqrt{m_1m_2}} \sqrt{\frac{G d\log d}{n}} - \lambda_1\frac{a_1+1}{a_1+\sigma_1(\hat{L})}\right)\|\hat{L}\|_* .
    \label{ineq:L2pi_sim}
\end{align}
Since $\sigma_1(\hat{L}) \leq \zeta \sqrt{m_1m_2}$, we take $\lambda_1 \gtrsim \frac{a_1+\zeta\sqrt{m_1m_2}}{a_1+1}\frac{(\sigma \vee \zeta)}{\sqrt{m_1m_2}} \sqrt{\frac{G d\log d}{n}}$, 
thus leading to
\begin{align*}
    \|(\hat{L} - L_0)_{\mathcal{I}}\|^2_{L_2(\Pi)} &\lesssim \beta \left\{(\sigma \vee \zeta)\sqrt{\frac{G d\log d}{n}} \frac{\|L_0\|_*}{\sqrt{m_1m_2}} + \lambda_1 \Phi_{a_1}(L_0)\right\} + \beta \Delta_{S_0}(N,m_1,m_2) + \zeta^2\sqrt{\frac{\log d}{d}},
\end{align*}
 which implies 
\begin{multline}\label{the2:ineq}
     \frac{\|\hat{L} - L_0\|_F^2}{m_1m_2} \lesssim 
     \nu \beta \left\{(\sigma \vee \zeta)\sqrt{\frac{G d\log d}{n}} \frac{\|L_0\|_*}{\sqrt{m_1m_2}} + \lambda_1 \Phi_{a_1}(L_0)\right\} +\nu\beta \Delta_{S_0}(N,m_1,m_2) \\ + \zeta^2\sqrt{\frac{\log d}{d}} + \frac{4\zeta^2 s_0}{m_1m_2}. 
\end{multline}





\textbf{Part 2:} 
We adopt the same projection definitions as those introduced in  \cite[Appendix C]{zhao2025noisy} to facilitate the derivation of the TL1 function, i.e., $\Phi_{a_1}(\cdot)$,  on low-rank matrices.
For any matrix $A \in  \mathbb{R}^{m_1 \times m_2}$, let $U_A$ and $V_A$ be the left and right singular matrices of $A$, and $D_A$ is the diagonal matrix with the singular values of $A$, i.e., the SVD of $A$ is expressed by  $ A= U_A D_A V_A^\intercal $. 
We define $S_U(A)$ and $S_V(A)$ to be the linear subspaces spanned by column vectors of $U_A$ and $V_A,$ respectively, and denote their corresponding orthogonal components, denoted by $S^\perp_{U}$ and $S^\perp_{V}$. 

For any matrix $B \in  \mathbb{R}^{m_1 \times m_2}$,
we set
\begin{align}
\mathcal{P}^\perp_{A}(B) = \mathbf P_{S^\perp_{U}(A)}B\mathbf P_{S^\perp_{V(A)}} \quad \mbox{and} \quad  \mathcal{P}_A(B) = B-\mathcal{P}^\perp_{A}(B), 
\end{align}
where $\mathbf P_S$ denotes the projection onto the linear subspace $S$. 
Then, by the Mean Value Theorem and  \cite[Lemma 4]{zhao2025noisy}, there exists a matrix $\tilde{L}$ componentwise between $\hat{L}$ and $L_0 + P_{L_0}^\perp (\hat{L} - L_0)$ such that
\begin{align}
    \Phi_{a_1}(\hat{L}) &= \Phi_{a_1}(L_0 + \hat{L} - L_0)  \notag\\
    &= \Phi_{a_1}(L_0 + \mathcal{P}_{L_0}^\perp (\hat{L} - L_0)) + \langle \nabla \Phi_{a_1}(\tilde{L}), \mathcal{P}_{L_0} (\hat{L} - L_0) \rangle \notag\\
    &\geq \Phi_{a_1}(L_0) + \Phi_{a_1}(\mathcal{P}_{L_0}^\perp (\hat{L} - L_0)) - \left\| \nabla \Phi_{a_1}(\tilde{L}) \right\| \left\| \mathcal{P}_{L_0} (\hat{L} - L_0) \right\|_* \notag\\
    &\geq \Phi_{a_1}(L_0) + \Phi_{a_1}(\mathcal{P}_{L_0}^\perp (\hat{L} - L_0)) - \frac{a_1(1+a_1)}{a_1^2} \left\| \mathcal{P}_{L_0} (\hat{L} - L_0) \right\|_* \notag\\
    &= \Phi_{a_1}(L_0) + \Phi_{a_1}(\mathcal{P}_{L_0}^\perp (\hat{L} - L_0)) - \frac{a_1+1}{a_1} \left\| \mathcal{P}_{L_0} (\hat{L} - L_0) \right\|_* .
\end{align}

Following the inequality \eqref{ineq:end-lemma4} in Lemma \ref{lemma_L2Pi}, we have
\begin{align}
    \|(\hat{L} - L_0)_{\mathcal{I}}\|^2_{L_2(\Pi)}\lesssim \beta\frac{\sigma}{\sqrt{m_1m_2}} \sqrt{\frac{G d\log d}{N}} \left(\left\| \mathcal{P}_{L_0} (\hat{L} - L_0) \right\|_* + \left\| \mathcal{P}_{L_0}^\perp (\hat{L} - L_0) \right\|_*\right) \notag\\
    + \beta \Delta_{S_0}(N,m_1,m_2) + \zeta^2 \sqrt{\frac{\log d}{n}}   
     +\beta\lambda_1 \Phi_{a_1} (L_0) \notag  \\
    - \beta\lambda_1 \left(\Phi_{a_1}(L_0) + \Phi_{a_1}(\mathcal{P}_{L_0}^\perp (\hat{L} - L_0)\right) 
    - \frac{a_1+1}{a_1} \left\| \mathcal{P}_{L_0} (\hat{L} - L_0) \right\|_*) \notag \\
    + \zeta \frac{\left\| \mathcal{P}_{L_0} (\hat{L} - L_0) \right\|_* + \left\| \mathcal{P}_{L_0}^\perp (\hat{L} - L_0) \right\|_*}{\sqrt{m_1m_2}} \sqrt{\frac{G M\log d}{n}}, \notag 
\end{align}
which can be written as
\begin{align}
    \|(\hat{L} - L_0)_{\mathcal{I}}\|^2_{L_2(\Pi)}\lesssim \beta \left(\frac{\sigma \vee \zeta}{\sqrt{m_1m_2}} \sqrt{\frac{G d\log d}{n}} + \lambda_1 \frac{a_1+1}{a_1} \right) \left\| \mathcal{P}_{L_0} (\hat{L} - L_0) \right\|_* \notag\\
    + \beta \Delta_{S_0}(N,m_1,m_2) + \zeta^2 \sqrt{\frac{\log d}{n}} \notag\\ 
    + \beta \left(\frac{\sigma \vee \zeta}{\sqrt{m_1m_2}} \sqrt{\frac{G d\log d}{n}} - \lambda_1\frac{a_1+1}{a_1+\zeta\sqrt{m_1m_2}}\right)\left\| \mathcal{P}_{L_0}^\perp (\hat{L} - L_0) \right\|_* .  
\end{align}

Since $L_0$ is exactly low-rank with rank $r_0$, then
\begin{align*}
    \|\mathcal{P}_{L_0} (\hat{L} -L_0)\|_* &\leq \sqrt{\rank (\mathcal{P}_{L_0} (\hat{L} -L_0))} \|\hat{L} - L_0\|_\infty \\
    &\leq \sqrt{2\;\rank (\hat{L} - L_0)} \|\hat{L} - L_0\|_F \leq  \sqrt{2r_0} \|\hat{L} - L_0\|_F.
\end{align*}

Taking $\lambda_1 \gtrsim \frac{a_1+\zeta\sqrt{m_1m_2}}{a_1+1}\frac{(\sigma \vee \zeta)}{\sqrt{m_1m_2}} \sqrt{\frac{G d\log d}{n}}$, 
we have 
\begin{align}
    \|(\hat{L} - L_0)_{\mathcal{I}}\|^2_{L_2(\Pi)}
    \lesssim \beta \left(\frac{\sigma \vee \zeta}{\sqrt{m_1m_2}} \sqrt{\frac{G d\log d}{n}} + \lambda_1 \frac{a_1+1}{a_1} \right) \sqrt{2r_0} \|\hat{L} - L_0\|_F \notag\\
    + \beta \Delta_{S_0}(N,m_1,m_2) + \zeta^2 \sqrt{\frac{\log d}{n}} \notag\\
    \lesssim \beta \lambda_1 \frac{a_1+1}{a_1}  \sqrt{2r_0} \|\hat{L} - L_0\|_F 
    + \beta \Delta_{S_0}(N,m_1,m_2) + \zeta^2 \sqrt{\frac{\log d}{n}}.
\end{align}

It follows from \eqref{ineq_L_fro} that we get
\begin{align}
    \frac{\|\hat{L} - L_0\|_F^2}{m_1m_2} &\lesssim \nu\beta^2 \lambda_1^2 \frac{(a_1+1)^2m_1m_2}{a_1^2}r_0
    + \nu \beta \Delta_{S_0}(N,m_1,m_2) + \nu\zeta^2 \sqrt{\frac{\log d}{n}} + \frac{4\zeta^2 s_0}{m_1m_2}. \label{ineq:exact_large_a}
\end{align}

Combining \textbf{Part 1} and \textbf{Part 2} yields
\begin{multline} \label{general_L_bound}
    \frac{\|\hat{L} - L_0\|_F^2}{m_1m_2} \lesssim \nu \beta\min\left\{(\sigma \vee \zeta)\sqrt{\frac{G d\log d}{n}} \frac{\|L_0\|_*}{\sqrt{m_1m_2}} + \lambda_1 \Phi_{a_1}(L_0), \beta \lambda_1^2 \frac{(a_1+1)^2m_1m_2}{a_1^2}r_0\right\}\\
    + \nu \beta \Delta_{S_0}(N,m_1,m_2) + \nu\zeta^2 \sqrt{\frac{\log d}{n}} + \frac{4\zeta^2 s_0}{m_1m_2}. 
\end{multline}




\end{proof}

\subsection{Proof of Corollary \ref{coro:approx_large_a} and Corollary \ref{coro:exact_lowrank}}\label{S3}

\begin{proof}[\textbf{Proof of Corollary \ref{coro:approx_large_a}}]
Take $\lambda_1 \asymp \frac{a_1+\zeta\sqrt{m_1m_2}}{a_1+1}\frac{(\sigma \vee \zeta)}{\sqrt{m_1m_2}} \sqrt{\frac{G d\log d}{n}}$,
 for any $a_1^{-1} = \mathcal{O}\left(\left(\zeta\sqrt{m_1m_2}\right)^{-1}\right)$, by using
 $\|L_0\|_*/\sqrt{m_1m_2} \leq \gamma$ and $\Phi_{a_1}(L_0) \leq \frac{a_1+1}{a_1}\|L_0\|_*$, the inequality \eqref{the2:ineq} becomes
    \begin{align*}
        (\sigma\vee\zeta)\sqrt{\frac{G d\log d}{n}} \frac{\|L_0\|_*}{\sqrt{m_1m_2}} + \frac{a_1+\zeta\sqrt{m_1m_2}}{a_1+1}\frac{(\sigma \vee \zeta)}{\sqrt{m_1m_2}} \sqrt{\frac{G d\log d}{n}} \frac{a_1+1}{a_1}\|L_0\|_* \\
        \lesssim (\sigma\vee\zeta)\gamma\sqrt{\frac{G d\log d}{n}}.
    \end{align*}
    Thus, we conclude that
    \begin{align*}
        \frac{\|\hat{L} - L_0\|_F^2}{m_1m_2} \lesssim \nu\beta(\sigma\vee \zeta) \gamma\sqrt{\frac{Gd\log d}{n}} + \nu \beta \Delta_{S_0}(N,m_1,m_2) + \nu\zeta^2 \sqrt{\frac{\log d}{n}} + \frac{4\zeta^2 s_0}{m_1m_2} ,
    \end{align*}
    which is the desired result. 
\end{proof}

\begin{proof}[\textbf{Proof of Corollary \ref{coro:exact_lowrank}}]
We discuss three scenarios as listed in Corollary \ref{coro:exact_lowrank} individually.

Scenario \textit{(i)}:
when $a_1^{-1} = \mathcal{O}\left(\left(\zeta\sqrt{m_1m_2}\right)^{-1}\right)$, the second term in \eqref{general_L_bound}, namely $\lambda_1^2 \frac{(a_1+1)^2m_1m_2}{a_1^2}r_0$, becomes smaller than the first term.  
Then, we have 
\begin{align*}
   & \frac{\|\hat{L} - L_0\|_F^2}{m_1m_2} \\
    \lesssim &\nu\beta^2 \left(\frac{a_1+\zeta\sqrt{m_1m_2}}{a_1+1}\frac{(\sigma \vee \zeta)}{\sqrt{m_1m_2}} \sqrt{\frac{Gd\log d}{n}}\right)^2 \frac{(a_1+1)^2m_1m_2}{a_1^2}r_0 \\
    &\qquad + \nu \beta \Delta_{S_0}(N,m_1,m_2) + \nu\zeta^2 \sqrt{\frac{\log d}{n}} + \frac{4\zeta^2 s_0}{m_1m_2}\\
    \lesssim& \nu\beta^2(\sigma \vee \zeta)^2 \frac{(a_1+\zeta\sqrt{m_1m_2})^2}{a_1^2}r_0 \frac{Gd\log d}{n}+ \nu \beta \Delta_{S_0}(N,m_1,m_2) + \nu\zeta^2 \sqrt{\frac{\log d}{n}} + \frac{4\zeta^2 s_0}{m_1m_2}\\
    \lesssim &\nu\beta^2(\sigma \vee \zeta)^2 r_0 \frac{Gd\log d}{n}+ \nu \beta \Delta_{S_0}(N,m_1,m_2) + \nu\zeta^2 \sqrt{\frac{\log d}{n}} + \frac{4\zeta^2 s_0}{m_1m_2}.
\end{align*}

By comparing the order of the components in the first term of \eqref{general_L_bound} which is $\sqrt{\frac{d\log d}{n}}$ and $\frac{(a_1+\zeta\sqrt{m_1m_2})^2}{a_1^2} \frac{d\log d}{n}$, we can further refine the admissible range for $a_1$, leading to the results below.

Scenario \textit{(ii)}: when $a_1^{-1} = \mathcal{O}\left(\left(\sqrt{m_1m_2}\left(d\log d / n\right)^{1/4}\right)^{-1}\right)$, we can verify that 
\begin{align}
    \left(\frac{a_1+\zeta\sqrt{m_1m_2}}{a_1}\right)^2 \frac{d\log d}{n} \lesssim \sqrt{\frac{d\log d}{n}},
\end{align}
thus leading to
\begin{multline}
    \frac{\|\hat{L} - L_0\|_F^2}{m_1m_2} 
    \lesssim \nu\beta^2(\sigma \vee \zeta)^2 \left(\frac{a_1+\zeta\sqrt{m_1m_2}}{a_1}\right)^2r_0 \frac{Gd\log d}{n}+ \nu\zeta^2 \sqrt{\frac{\log d}{n}} \\
    + \nu \beta \Delta_{S_0}(N,m_1,m_2)+ \frac{4\zeta^2 s_0}{m_1m_2}.
\end{multline}

Scenario \textit{(iii)}: when $a_1 = \mathcal{O}\left(\sqrt{m_1m_2}\left(d\log d / n\right)^{1/4}\right)$, we have
\begin{align}
    \left(\frac{a_1+\zeta\sqrt{m_1m_2}}{a_1}\right)^2 \frac{d\log d}{n} \gtrsim \sqrt{\frac{d\log d}{n}},
\end{align}
and hence we get
\begin{align*}
    \frac{\|\hat{L} - L_0\|_F^2}{m_1m_2} \leq \nu\beta(\sigma\vee \zeta) r_0\sqrt{\frac{Gd\log d}{n}} +  \nu \beta \Delta_{S_0}(N,m_1,m_2) + \nu\zeta^2 \sqrt{\frac{\log d}{n}} + \frac{4\zeta^2 s_0}{m_1m_2}.
\end{align*}
\end{proof}

\subsection{Proof of Theorem \ref{thm:sparsity} and Corollary \ref{cor:sparsity}}\label{S4}

\begin{proof}[\textbf{Proof of Theorem \ref{thm:sparsity}}]
    Assume the support of $S_0$ has cardinality $s_0$ and the support of $\hat{S}$ has cardinality $\hat{s}$. Let $s'$ denote the number of indices $(k, l) \notin \tilde{\mathcal{I}}$ for which $\hat{S}_{kl} \neq 0$. Then, the total support size satisfies $\hat{s} = s_0 + s'$.
    
   It follows from $Q(\hat{L}, \hat{S}) \leq Q(\hat{L}, S_0)$ and the inequality \eqref{basic_ineq_for_S} that
    \begin{align*}
    \frac{2}{N} \sum_{i=1}^{N} \langle T_i, \hat{L} - L_0 \rangle \langle T_i, \hat{S} - S_0 \rangle \leq \frac{2\sigma}{N} \sum_{i=1}^{N} \langle T_i \xi_i , \hat{S} - S_0  \rangle  + \lambda_2\left(\phi_{a_2} (S_0) - \phi_{a_2} (\hat{S})\right),
    \end{align*}
    which implies
    \begin{align*}
        \lambda_2\phi_{a_2}(\hat{S}) &\leq \frac{2}{N} \sum_{i=1}^{N} |\langle T_i, \hat{L} -L_0 \rangle| |\langle T_i, \hat{S} -S_0 \rangle| + \frac{2\sigma}{N} \sum_{i=1}^{N} \langle T_i \xi_i , \hat{S} -S_0  \rangle + \lambda_2\phi_{a_2}(S_0) \notag\\
        &\leq \frac{1}{N} \sum_{i=1}^N \langle T_i, \hat{L} - L_0 \rangle^2 + \frac{1}{N} \sum_{i=1}^N \langle T_i, \hat{S} - S_0 \rangle^2 + \frac{2\sigma}{N} \sum_{i \in \Omega} \langle T_i \xi_i , \hat{S} -S_0  \rangle \notag\\
        &\qquad+ \frac{2\sigma}{N} \sum_{i \in \tilde{\Omega}} \langle T_i \xi_i , \hat{S} -S_0  \rangle + \lambda_2\phi_{a_2}(S_0) \notag\\
        &\lesssim \frac{1}{N} \sum_{i=1}^N \langle T_i, \hat{L} - L_0 \rangle^2 + \frac{1}{N} \sum_{i=1}^N \langle T_i, \hat{S} - S_0 \rangle^2 \notag\\
        &\quad+ \|\Sigma\|_\infty\sqrt{2s_0 + s'}\|\hat{S} - S_0\|_F + (\sigma\vee\zeta)^2\frac{\log d}{N} + \lambda_2\phi_{a_2}(S_0)\notag \\
        &\lesssim \Delta_{L_0}(n, m_1, m_2) + \Delta_{S_0}(N, m_1, m_2) + \frac{\log d}{N}(\sqrt{s_0}+\sqrt{s'})\Delta_S + \lambda_2\phi_{a_2}(S_0) . \notag\\
    \end{align*}
    Therefore, we obtain
    \begin{align}
        \phi_{a_2}(\hat{S}) &\lesssim \lambda_2^{-1} \Delta_{L_0}(n, m_1, m_2) + \Delta_{S_0}(N, m_1, m_2) + \phi_{a_2}(S_0) \notag\\
        &\qquad + \lambda_2^{-1}(\sqrt{s_0}+\sqrt{s'})\sqrt{N\Delta_{S_0}(N, m_1, m_2)} \frac{\log d}{N}.
    \end{align}
    
    Since $\phi_{a_2}$ is an increasing function,  when $a_2 = {\scriptstyle \mathcal{O}}\left(\lambda_2(\Delta_{L_0}(n, m_1, m_2) + \Delta_{S_0}(N, m_1, m_2))^{-1/2}\right)$, we obtain $|\hat{S}_{kl}| \geq c'\sqrt{\lambda_2(a_2^2+a_2)}(\Delta_{L_0}(n, m_1, m_2) + \Delta_{S_0}(N, m_1, m_2))^{-1/4}$ for $(k, l) \notin \tilde{\mathcal{I}}$, by Lemma \ref{lemma_for_sparsity}. We can further get
    \begin{align*}
        \phi_{a_2}(\hat{S}) &= \sum_{i,j} \frac{(a_2+1) |\hat{S}_{ij}|}{a_2+|\hat{S}_{ij}|} \geq \sum_{(k,l) \notin \mathcal{\tilde{I}}} \frac{(a_2+1) |\hat{S}_{kl}|}{a_2+|\hat{S}_{kl}|}\\
        &\gtrsim s' \frac{(a_2 + 1) \sqrt{\lambda_2(a_2^2+a_2)}(\Delta_{L_0}(n, m_1, m_2) + \Delta_{S_0}(N, m_1, m_2))^{-1/4}}{a_2 +\sqrt{\lambda_2(a_2^2+a_2)}(\Delta_{L_0}(n, m_1, m_2) + \Delta_{S_0}(N, m_1, m_2))^{-1/4}},
    \end{align*}
    then,
    \begin{align}
        s' &\lesssim \frac{a_2+\sqrt{\lambda_2(a_2^2+a_2)}(\Delta_{L_0}(n, m_1, m_2) + \Delta_{S_0}(N, m_1, m_2))^{-1/4}}{\sqrt{\lambda_2(a_2^2+a_2)}(\Delta_{L_0}(n, m_1, m_2) + \Delta_{S_0}(N, m_1, m_2))^{-1/4}} \frac{1}{a_2+1} \times\notag\\
        &\qquad\qquad \bigg\{ \lambda_2^{-1} (\Delta_{L_0}(n, m_1, m_2) + \Delta_{S_0}(N, m_1, m_2)) + \lambda_2^{-1}(\sqrt{s_0}+\sqrt{s'})\Delta_S \frac{\log d}{N} + \phi_{a_2}(S_0) \bigg\} \notag\\
        &\lesssim \left(\frac{a_2}{\sqrt{\lambda_2(a_2^2+a_2)}(\Delta_{L_0}(n, m_1, m_2) + \Delta_{S_0}(N, m_1, m_2))^{-1/4}} + 1\right) \frac{1}{a_2+1} \times\notag \\
        &\qquad\qquad \Bigg\{ \lambda_2^{-1} (\Delta_{L_0}(n, m_1, m_2) + \Delta_{S_0}(N, m_1, m_2)) + \lambda_2^{-1} \sqrt{s_0}\sqrt{N \Delta_{S_0}(N, m_1, m_2)} \frac{\log d}{N} \notag\\
        &\qquad\qquad\qquad + \phi_{a_2}(S_0)  + \sqrt{s'}\lambda_2^{-1} \sqrt{N\Delta_{S_0}(N, m_1, m_2)}\frac{\log d}{N} \Bigg\}. \label{ineq:s'-last}
    \end{align}

    For convenience, we denote the first term in \eqref{ineq:s'-last} by
    \begin{align*}
        \Theta_2 := \frac{a_2}{\sqrt{\lambda_2(a_2^2 + a_2)}\,(\Delta_{L_0}(n, m_1, m_2) + \Delta_{S_0}(N, m_1, m_2))^{-1/4}} + 1,
    \end{align*}
     which is less than 2, since $a_2 / \sqrt{\lambda_2(a_2^2+a_2)}(\Delta_{L_0}(n, m_1, m_2) + \Delta_{S_0}(N, m_1, m_2))^{-1/4} < 1$ when $a_2 = {\scriptstyle \mathcal{O}}\left(\lambda_2(\Delta_{L_0}(n, m_1, m_2) + \Delta_{S_0}(N, m_1, m_2))^{-1/2}\right)$.
Furthermore, we decompose the inequality \eqref{ineq:s'-last} into three groups:

    \begin{align}
     s' &\lesssim \frac{\Theta_2}{a_2+1}(\Delta_{L_0}(n, m_1, m_2) + \Delta_{S_0}(N, m_1, m_2))\lambda_2^{-1}, \notag\\
        s'&\lesssim \frac{\Theta_2}{a_2+1} \left(\lambda_2^{-1} \sqrt{s_0}\sqrt{N\Delta_{S_0}(N, m_1, m_2)} \frac{\log d}{N} + \phi_{a_2}(S_0)\right), \notag\\
        s' &\lesssim \left(\frac{\Theta_2}{a_2+1}\right)^2 \frac{\log^2 d}{N} \Delta_{S_0}(N, m_1, m_2) \lambda_2^{-2}. \notag
    \end{align}
    


Therefore, we obtain that
    \begin{multline}
        \hat{s} = s_0 + s' \lesssim s_0 +\frac{\Theta_2}{a_2+1} \max\Bigg\{ \lambda_2^{-1} \sqrt{s_0}\sqrt{N \Delta_{S_0}(N, m_1, m_2)} \frac{\log d}{N} + \phi_{a_2}(S_0), \\ \frac{\Theta_2}{a_2+1} \frac{\log^2 d}{N} \Delta_{S_0}(N, m_1, m_2) \lambda_2^{-2} + (\Delta_{L_0}(n, m_1, m_2) + \Delta_{S_0}(N, m_1, m_2))\lambda_2^{-1} \Bigg\}. \label{sparsity_bound}
    \end{multline}

    
    Moreover, taking $\lambda_2 \asymp \Delta_{L_0}(n, m_1, m_2) + \Delta_{S_0}(N, m_1, m_2) $ implies that $ \lambda^{-1} \log d / N = {\scriptstyle \mathcal{O}}(1) $ and $\sqrt{s_0}\sqrt{N \Delta_{S_0}(N, m_1, m_2)} = {\scriptstyle \mathcal{O}}(s_0)$. Then, the first term of $s'$ in \eqref{sparsity_bound} becomes ${\scriptstyle \mathcal{O}}(s_0) $ with high probability when $ a_2 $ is sufficiently small, and the second term is of a constant order as well. Combining both, we obtain $ s' = {\scriptstyle \mathcal{O}}(s_0) $ with high probability and hence we conclude that $ \hat{s} = s' + s_0 = \mathcal{O}_p(s_0) $.

    
    
\end{proof}

\begin{proof}[\textbf{Proof of Corollary \ref{cor:sparsity}}]
    Combining the results from Scenario (i) in Corollary \ref{coro:exact_lowrank}, Theorem \ref{thm:gen_upp_S} and Theorem \ref{thm:sparsity}, by taking $a_1^{-1} = \mathcal{O}((\sqrt{m_1m_2})^{-1})$, $\lambda_1 \asymp \frac{(\sigma \vee \zeta)}{\sqrt{m_1m_2}}  \sqrt{\frac{Gd\log d}{n}}$, $\lambda_2 \asymp  (\sigma\vee\zeta)\frac{\log d}{N}$, and $a_2 = {\scriptstyle \mathcal{O}}\left(\sqrt{\frac{d\log d}{n}}\right)$, we obtain that $\|\hat{S}\|_0 = \mathcal{O}_p(s_0)$ and
    \begin{align*}
        \frac{\|\hat{L} - L_0\|_F^2}{m_1m_2}+ \frac{\|\hat{S} - S_0\|_F^2}{m_1m_2} &\lesssim \beta^2(\sigma\vee\zeta)^2 r_0 \frac{Gd\log d}{n} + \zeta^2 \sqrt{\frac{\log d}{n}} + \frac{\zeta^2s_0}{m_1m_2} \\
        &\qquad\qquad+ \beta \Delta_{S_0}(N,m_1,m_2)  + \frac{N}{m_1m_2} \Delta_{S_0}(N,m_1,m_2)\\
        &= \mathcal{O}_p\left(r_0\frac{d\log d}{n} + \frac{s_0\log d}{m_1m_2}\right).
    \end{align*}
\end{proof}

\subsection{Proof of Theorem \ref{thm:lowrankness} and Corollary \ref{cor:LR}}\label{S5}

\begin{proof}[\textbf{Proof of Theorem \ref{thm:lowrankness}}]
    It follows from the  inequality: $Q(\hat{L}, \hat{S}) \leq Q(L_0, \hat{S})$ together with \eqref{basic_ineq_L} that
    \begin{align}
        \frac{2}{N} \sum_{i=1}^{N} \langle T_i, \hat{L} -L_0 \rangle \langle T_i, \hat{S} -S_0 \rangle 
        \leq \frac{2\sigma}{N} \sum_{i=1}^{N} \langle T_i \xi_i , \hat{L} -L_0  \rangle  + \lambda_1 \Phi_{a_1} (L_0) - \lambda_1 \Phi_{a_1} (\hat{L}).
    \end{align}
   Let  $\hat{r}$ be the rank of $\hat{L}$. By a series of calculations,
    \begin{align*}
        \lambda_1 \Phi_{a_1} (\hat{L}) &\leq \frac{2}{N} \sum_{i=1}^{N} |\langle T_i, \hat{L} -L_0 \rangle| |\langle T_i, \hat{S} -S_0 \rangle| + \frac{2\sigma}{N} \sum_{i=1}^{N} \langle T_i \xi_i , \hat{L} -L_0  \rangle  + \lambda_1 \Phi_{a_1} (L_0) \notag\\
        &\leq \frac{1}{N} \sum_{i=1}^N \langle T_i, \hat{L} - L_0 \rangle^2 + \frac{1}{N} \sum_{i=1}^N \langle T_i, \hat{S} - S_0 \rangle^2  \\
        &\qquad + \frac{2\sigma}{N} \sum_{i \in \Omega} \langle T_i \xi_i , \hat{L} -L_0  \rangle + \frac{2\sigma}{N} \sum_{i \in \tilde{\Omega}} \langle T_i \xi_i , \hat{L} -L_0  \rangle + \lambda_1 \Phi_{a_1} (L_0) \\
        &\lesssim \frac{1}{N} \sum_{i=1}^N \langle T_i, \hat{L} - L_0 \rangle^2 + \frac{1}{N} \sum_{i=1}^N \langle T_i, \hat{S} - S_0 \rangle^2  \\
        &\qquad + 2\|\Sigma\|\sqrt{\rank(\hat{L} - L_0)}\|\hat{L} - L_0\|_F + \lambda_1 \Phi_{a_1} (L_0) \\
        &\lesssim \Delta_{L_0}(n, m_1, m_2) + \Delta_{S_0}(N, m_1, m_2) \\
        &\qquad+ \sqrt{\frac{Gd\log d}{Nm_1m_2}} (\sqrt{r_0}+\sqrt{\hat{r}}) \sqrt{m_1m_2\Delta_{L_0}(n, m_1, m_2)} + \lambda_1\Phi_{a_1}(L_0),
    \end{align*}
     we obtain 
    \begin{align*}
     \Phi_{a_1} (\hat{L}) 
        &\lesssim \lambda_1^{-1}(\Delta_{L_0}(n, m_1, m_2) + \Delta_{S_0}(N, m_1, m_2)) \\
        &\qquad+ \lambda_1^{-1}\sqrt{\frac{Gd\log d}{N}} (\sqrt{r_0}+\sqrt{\hat{r}})\sqrt{\Delta_{L_0}(n, m_1, m_2)} + \Phi_{a_1}(L_0).
    \end{align*}

    When $a_1 = {\scriptstyle \mathcal{O}} \left((a_1+1)\lambda_1\left((\Delta_{L_0}(n, m_1, m_2) + \Delta_{S_0}(N, m_1, m_2))^2 Gd\log d/( nm_1m_2)\right)^{-1/4}\right)$,  Lemma \ref{lemma_for_lowrankness} implies that the smallest singular value of $\hat{L}$ is at least $$c \sqrt{\lambda_1(a_1^2+a_1)}\left((\Delta_{L_0}(n, m_1, m_2) + \Delta_{S_0}(N, m_1, m_2))^2 Gd\log d/( nm_1m_2)\right)^{-1/8}.$$ 
    It further follows from the increasing property of function $\Phi_{a_1}(\cdot)$ that 
    \begin{align*}
        \Phi_{a_1}(\hat{L}) &= \sum_{j=1}^m\frac{(a_1+1)\sigma_j}{a_1+\sigma_j} \\
        &\gtrsim \hat{r}\frac{(a_1+1)\sqrt{\lambda_1(a_1^2+a_1)}\left((\Delta_{L_0}(n, m_1, m_2) + \Delta_{S_0}(N, m_1, m_2))^2 Gd\log d/( nm_1m_2)\right)^{-1/8}}{a_1 + \sqrt{\lambda_1(a_1^2+a_1)}\left((\Delta_{L_0}(n, m_1, m_2) + \Delta_{S_0}(N, m_1, m_2))^2 Gd\log d/( nm_1m_2)\right)^{-1/8}}.
    \end{align*}
    Therefore, we have
    \begin{align}
        \hat{r} &\lesssim \frac{a_1 + \sqrt{\lambda_1(a_1^2+a_1)}\left((\Delta_{L_0}(n, m_1, m_2) + \Delta_{S_0}(N, m_1, m_2))^2 Gd\log d/( nm_1m_2)\right)^{-1/8}}{\sqrt{\lambda_1(a_1^2+a_1)}\left((\Delta_{L_0}(n, m_1, m_2) + \Delta_{S_0}(N, m_1, m_2))^2 Gd\log d/( nm_1m_2)\right)^{-1/8}}\frac{1}{a_1+1}  \notag\\ 
        &\qquad \times \Bigg\{ \lambda_1^{-1}(\Delta_{L_0}(n, m_1, m_2) + \Delta_{S_0}(N, m_1, m_2)) \notag\\
         &\qquad\qquad + \lambda_1^{-1}\sqrt{\frac{Gd\log d}{N}} (\sqrt{r_0}+\sqrt{\hat{r}}) \sqrt{\Delta_{L_0}(n, m_1, m_2)}+ \Phi_{a_1}(L_0)
        \Bigg\} \notag\\
        &\lesssim \left(\frac{a_1}{\sqrt{\lambda_1(a_1^2+a_1)}\left((\Delta_{L_0}(n, m_1, m_2) + \Delta_{S_0}(N, m_1, m_2))^2 d\log d/( nm_1m_2)\right)^{-1/8}} + 1\right) \frac{1}{a_2+1} \notag\\
        &\qquad \times\Bigg\{ \lambda_1^{-1}(\Delta_{L_0}(n, m_1, m_2) + \Delta_{S_0}(N, m_1, m_2)) + \lambda_1^{-1}\sqrt{\frac{Gd\log d}{N}} \sqrt{r_0} \sqrt{\Delta_{L_0}(n, m_1, m_2)} \notag\\
        &\qquad\qquad + \lambda_1^{-1}\sqrt{\frac{Gd\log d}{N}} \sqrt{\hat{r}} \sqrt{\Delta_{L_0}(n, m_1, m_2)} + \Phi_{a_1}(L_0) \Bigg\}. \label{ineq:hat_r}
    \end{align}

    For convenience, we define:
    \begin{align*}
        \Theta_1 := \frac{a_1}{\sqrt{\lambda_1(a_1^2+a_1)}\left((\Delta_{L_0}(n, m_1, m_2) + \Delta_{S_0}(N, m_1, m_2))^2 Gd\log d/( nm_1m_2)\right)^{-1/8}} + 1.
    \end{align*}
When $a_1 = {\scriptstyle \mathcal{O}} \left((a_1+1)\lambda_1\left((\Delta_{L_0}(n, m_1, m_2) + \Delta_{S_0}(N, m_1, m_2))^2 Gd\log d/( nm_1m_2)\right)^{-1/4}\right)$, we can verify that $\Theta_1 < 2$. 
Then inequality \eqref{ineq:hat_r} can be simplified as,
     \begin{align}
         \hat{r} &\lesssim \frac{\Theta_1}{a_1+1} \Bigg\{ \lambda_1^{-1}(\Delta_{L_0}(n, m_1, m_2) + \Delta_{S_0}(N, m_1, m_2))
        + \lambda_1^{-1}\sqrt{\frac{Gd\log d}{N}} \sqrt{r_0} \sqrt{\Delta_{L_0}(n, m_1, m_2)} \notag\\
        &\qquad\qquad\qquad + \lambda_1^{-1}\sqrt{\frac{Gd\log d}{N}} \sqrt{\hat{r}} \sqrt{\Delta_{L_0}(n, m_1, m_2)} + \Phi_{a_1}(L_0) \Bigg\}, \notag 
     \end{align}
     which can be further split into 
     \begin{align}
         \hat{r}& \lesssim \frac{\Theta_1}{a_1+1} \lambda_1^{-1} (\Delta_{L_0}(n, m_1, m_2) + \Delta_{S_0}(N, m_1, m_2)), \notag\\
         \hat{r} &\lesssim \frac{\Theta_1}{a_1+1} \left(\lambda_1^{-1}\sqrt{\frac{Gd\log d}{N}} \sqrt{\Delta_{L_0}(n, m_1, m_2)} \sqrt{r_0} + \Phi_{a_1}(L_0)\right), \notag\\
         \hat{r}& \lesssim \frac{\Theta_1^2}{(a_1+1)^2} \lambda_1^{-2}\Delta_{L_0}(n, m_1, m_2)\frac{Gd\log d}{N}. \notag         
     \end{align}

     Therefore, we have
     \begin{multline}
         \hat{r} \lesssim \frac{\Theta_1}{a_1+1} \max\Bigg\{ \lambda_1^{-1}\sqrt{\frac{Gd\log d}{N} \Delta_{L_0}(n, m_1, m_2)} \sqrt{r_0} + \Phi_{a_1}(L_0), \\
         \lambda_1^{-1} (\Delta_{L_0}(n, m_1, m_2) + \Delta_{S_0}(N, m_1, m_2))+ \frac{\Theta_1}{a_1+1}\lambda_1^{-2}\Delta_{L_0}(n, m_1, m_2)\frac{Gd\log d}{N}\Bigg\}. \label{gen_hat_r_bound}
     \end{multline}
\end{proof}

\begin{proof}[\textbf{Proof of Corollary \ref{cor:LR}}]
  Taking $ a_1 ={\scriptstyle {\mathcal{O}} }\left((m_1 m_2)^{1/4}\right)$, $\lambda_1 \asymp \frac{(\sigma \vee \zeta)}{\sqrt{m_1m_2}} \frac{a_1 + \zeta\sqrt{m_1m_2}}{a_1+1} \sqrt{\frac{Gd\log d}{n}}$, $\lambda_2 \asymp (\sigma\vee\zeta) \frac{d\log d}{N}$ and $a_2 = {\scriptstyle \mathcal{O}}\left(\left(\frac{d\log d}{N}\right)^{1/4}\right)$, we combine the results from Scenario \textit{(iii)} 
    in Corollary \ref{coro:exact_lowrank}, Theorem \ref{thm:gen_upp_S} and Theorem \ref{cor:LR} to obtain $\|\hat{S}\|_0 = \mathcal{O}_p(s_0)$. Furthermore, the first term in \eqref{gen_hat_r_bound} is $\mathcal{O}_p(r_0)$ and the second term is $\mathcal{O}_p(1)$, which implies that $\hat{r} = \mathcal{O}_p(r_0)$. 
    In addition, we have
    \begin{align*} 
        \frac{\|\hat{L} - L_0\|_F^2}{m_1m_2}+ \frac{\|\hat{S} - S_0\|_F^2}{m_1m_2} &\lesssim \beta(\sigma\vee \zeta) r_0\sqrt{\frac{Gd\log d}{n}} + \zeta^2 \sqrt{\frac{\log d}{n}} + \frac{\zeta^2s_0}{m_1m_2} \\
        &\qquad\qquad+ \beta \Delta_{S_0}(N,m_1,m_2)  + \frac{N}{m_1m_2} \Delta_{S_0}(N,m_1,m_2)\\
        &= \mathcal{O}_p\left(r_0\sqrt{\frac{d\log d}{n}} + \frac{s_0\log d}{m_1m_2}\right).
    \end{align*}
       
\end{proof}

\bibliographystyle{unsrt}
\bibliography{jml_references}

\end{document}

%% file: admm.tex
\subsection{Numerical algorithm}
\label{algorithm}
We apply the Alternating Direction Method of Multiplier (ADMM) \cite{boydPCPE11admm} for solving the proposed TL1-regularized model \eqref{est:hat_LS} due to its simplicity and efficiency. 
To this end, we introduce two auxiliary matrices, $J, R\in \mathbb{R}^{m_1 \times m_2},$ and rewrite the optimization problem \eqref{est:hat_LS} into an equivalent form, 
\begin{equation}
\begin{array}{l}
\displaystyle \min_{L,S,J,R} \frac{1}{N} \| Y - T \circ (J+R) \|_F^2 + \lambda_1\Phi_{a_1}  (L) + \lambda_2 \phi_{a_2} (S) \\
\text{subject to }\quad  J = L,\quad R = S,\quad \| J \|_{\infty} \leq \zeta,\quad \| R \|_{\infty} \leq \zeta,
\end{array}
\label{eq:full_optimization}
\end{equation}
where  \( T = \sum_{i=1}^{N} T_i \) and the symbol \( \circ \) denotes the elementwise Hadamard product. 
The augmented Lagrangian function corresponding to \eqref{eq:full_optimization} can be written as  
\begin{equation}\label{eq:augLagrangian}
    \begin{aligned}
    \mathcal{L}(L, S, J,R;B,D) 
=&  \frac{1}{N} \|Y - T \circ(J + R)\|_F^2 + \lambda_1\Phi_{a_1}  (L) + \lambda_2 \phi_{a_2} (S)\\
 & \qquad + \frac{\rho_1}{2} \|L - J + B\|_F^2 + \frac{\rho_2}{2} \|S - R + D\|_F^2,
\end{aligned}
\end{equation}
where \( B, D \in \mathbb{R}^{m_1 \times m_2} \) are dual variables and \( \rho_1, \rho_2 \) are positive parameters. The ADMM scheme involves iteratively minimizing the augmented Lagrangian \eqref{eq:augLagrangian} with respect to one variable at a time while keeping the rest fixed. 

Specifically, the $L$-subproblem is equivalent to
\begin{equation}\label{eq:Lsub}
 L^{k+1} = \argmin_{L} \mathcal{L}(L,S^k, J^k,R^k;B^k,D^k) =\argmin_{L}  \lambda_1 \Phi_{a_1} (L)
 + \frac{\rho_1}{2} \|L - J^k + B^k\|_F^2 ,
\end{equation}
with a closed-form solution $
    L^{k+1} = U\mbox{diag}\left(\{\text{prox}^{\text{TL1}}_{a_1}(\sigma_k, \lambda_1/\rho_1)\}_{1\leq k\leq m} \right)V^\intercal, 
$ 
where the Singular Value Decomposition (SVD) of the matrix $J^{k} -B^k = U\Sigma V^\intercal$, the diagonal matrix $\Sigma$ has elements $\sigma_k$ for $1\leq k\leq m$ with  $m = \min(m_1, m_2)$, and the proximal operator for the TL1 regularization \cite{zhang2015transformed} is defined as
\begin{align}
     \text{prox}^{\text{TL1}}_{a}(x, \mu) := & \argmin_{z\in\mathbb R}\left\{\mu\frac{(a+1)z}{a+z} + \frac 1 2 (z-x)^2\right\}\notag\\
    =& \mathrm{sign}(x) \left\{ \frac{2}{3}(a+|x|)\cos(\frac{\varphi(x)}{3}) - \frac{2a}{3}+\frac{|x|}{3} \right\},\label{eq:prox-TL1}
\end{align}
with $\varphi(x) = \arccos(1- {27\mu a(1+a)}/[{2(a+|x|)^3}]).$ 
 
The $S$-subproblem can be formulated by
\begin{equation} \label{eq:Ssub}
 S^{k+1} = \argmin_{S} \mathcal{L}(L^{k+1},S, J^k,R^k;B^k,D^k)=\argmin_{S} \lambda_2 \phi_{\alpha_2}(S) 
 + \frac{\rho_2}{2} \|S - R^k + D^k\|_F^2 ,
\end{equation}
which can be updated via
$
S^{k+1}=\text{prox}^{\text{TL1}}_{a_2} \left( R^k-D^k, \lambda_2/\rho_2 \right), 
$ 
with $\text{prox}^{\text{TL1}}_{a}$ defined in \eqref{eq:prox-TL1} and applied to each element of the matrix $R^k-D^k$ componentwise.

The $J$-subproblem can be written as
\begin{eqnarray}
J^{k+1} &=& \argmin_{\| J \|_{\infty} \leq \zeta}  \mathcal{L}(L^{k+1},S^{k+1}, J,R^k;B^k,D^k)\notag\\
\label{eq:Jsub2}
&=
&\argmin_{\| J \|_{\infty} \leq \zeta}   \frac{1}{N} \|Y - T \circ(J + R^k)\|_F^2 
 + \frac{\rho_1}{2} \|L^{k+1} - J + B^k\|_F^2 .
\end{eqnarray}
Ignoring the constraint of $\| J \|_{\infty} \leq \zeta$, we take the derivative of the objective function in \eqref{eq:Jsub2} with respect to $J$ and set it to zero, thus leading to the optimal solution
\begin{equation}
J^{k+\frac 1 2} := \left( \frac{2}{N} T \circ (Y-R^k) + \rho_1 (L^{k+1}+B^k) \right) \oslash \left( \frac{2}{N} {T} + \rho_1 I_d\right),
\end{equation}
where $\oslash$ denotes the elementwise division and $I_d$ denotes the identity matrix.  Then we project the solution to the constraint $[-\zeta, \zeta]$ by
$
J^{k+1} = \min \left\{ \max\Big\{ J^{k+\frac{1}{2}}, \zeta \Big\}, -\zeta \right\},
$ 
where $\min$ and $\max$ are conducted elementwise. 

Similarly,  the $R$-subproblem has a closed-form solution given by
\begin{equation}
R^{k+1} :=  \min \left\{ \max\Big\{ \big( \frac{2}{N} T \circ (Y-J^{k+1}) + \rho_2 (S^{k+1}+D^k) \big) \oslash \big( \frac{2}{N} {T} + \rho_2 I_d\big), \zeta \Big\}, -\zeta \right\}.
\end{equation}